\documentclass[a4paper,UKenglish,cleveref, autoref, thm-restate]{lipics-v2021}
\nolinenumbers
\hideLIPIcs

\usepackage{algorithm}
\usepackage{algorithmicx}
\usepackage[noend]{algpseudocode}

\usepackage{adjustbox}         % better centering
\usepackage[table]{xcolor}
\usepackage{multirow,siunitx}
\usepackage{caption}
\usepackage{etoolbox}

\usepackage{color}
\usepackage{url}
\usepackage{stmaryrd}
\usepackage{soul}

\usepackage{pgf}

\usepackage{amsmath,amssymb,amsthm} %latexsym,mathtools,dsfont
\usepackage{latexsym,bm,amsbsy}
\usepackage{enumitem}
\usepackage{booktabs,multirow}
\usepackage{xspace}
\usepackage{stmaryrd}

\usepackage{tikz}
\usepackage{forest}

\usetikzlibrary{arrows,decorations.pathmorphing,decorations.footprints,fadings,calc,trees,mindmap,shadows,decorations.text,patterns,positioning,shapes,matrix,fit}
\usetikzlibrary{arrows,shadows,backgrounds}
\usetikzlibrary{arrows.meta}
\usetikzlibrary{positioning,fit}
\usetikzlibrary{automata}
\usetikzlibrary{matrix}
\usetikzlibrary{shapes.symbols,shapes.misc,shapes.arrows,shapes.geometric}
\usetikzlibrary{matrix,chains,scopes,decorations.pathmorphing}

\title{Efficient Explanations for \\Knowledge Compilation Languages} %TODO Please add

\titlerunning{Efficient Explanations for KC Languages} %TODO Please add
\author{Xuanxiang Huang}{Universit\'{e} de Toulouse, Toulouse,
  France}{xuanxiang.huang@univ-toulouse.fr}{https://orcid.org/0000-0002-3722-7191}{} %[funding]

\author{Yacine Izza}{Universit\'{e} de Toulouse, Toulouse,
  France}{yacine.izza@univ-toulouse.fr}{https://orcid.org/0000-0002-7774-1945}{} %[funding]

\author{Alexey Ignatiev}{Monash University, Melbourne,
  Australia}{alexey.ignatiev@monash.edu}{https://orcid.org/0000-0002-4535-2902}{} %[funding]

\author{Martin C. Cooper}{Universit\'{e} Paul Sabatier, IRIT,
  Toulouse, France}{martin.cooper@irit.fr}{https://orcid.org/0000-0003-4853-053X}{} %[funding]

\author{Nicholas Asher}{IRIT, CNRS, Toulouse,
  France}{nicholas.asher@irit.fr}{https://orcid.org/0000-0002-7689-8246}{} %funding]

\author{Joao Marques-Silva}{IRIT, CNRS, Toulouse,
  France}{joao.marques-silva@irit.fr}{https://orcid.org/0000-0002-6632-3086}{} % [orcid]%[funding]

\authorrunning{X. Huang, Y.\ Izza, A.\ Ignatiev, M.\ C.\ Cooper,
  N.\ Asher and J.\ Marques-Silva}
\Copyright{X. Huang, Y.\ Izza, A.\ Ignatiev, M.\ C.\ Cooper, N.\ Asher, J.\ Marques-Silva}
\ccsdesc[500]{Theory of computation~Automated reasoning}

\keywords{Machine Learning, Explainable AI, Knowledge Compilation,
  Tractability} %TODO mandatory; please add comma-separated list of keywords

\category{} %optional, e.g. invited paper

\relatedversion{} %optional, e.g. full version hosted on arXiv, HAL, or other respository/website
\funding{%
This work was supported by the AI Interdisciplinary Institute ANITI, funded by
the French program ``Investing for the Future -- PIA3'' under Grant
agreement no. ANR-19-PI3A-0004, and by the H2020-ICT38 project COALA
``Cognitive Assisted agile manufacturing for a Labor force supported
by trustworthy Artificial intelligence''.
}

\EventShortTitle{CoRR 2021}
\EventAcronym{CoRR}
\EventYear{2021}
\SeriesVolume{01}
\ArticleNo{001}
\newcommand{\todoF}[2]{}

\newcommand{\fml}[1]{{\mathcal{#1}}}

\newcommand{\tn}[1]{\textnormal{#1}}
\newcommand{\tbf}[1]{\textnormal{\bfseries #1}}

\newcommand{\mbf}[1]{\ensuremath\mathbf{#1}}
\newcommand{\mbb}[1]{\ensuremath\mathbb{#1}}

\newcommand{\nnf}{\ensuremath\tn{NNF}}
\newcommand{\dnnf}{\ensuremath\tn{DNNF}}
\newcommand{\ddnnf}{\ensuremath\tn{d-DNNF}}
\newcommand{\sddnnf}{\ensuremath\tn{sd-DNNF}}
\newcommand{\decnnf}{\ensuremath\tn{dec-DNNF}}

\newcommand{\biglor}{\ensuremath\bigvee}
\newcommand{\bigland}{\ensuremath\bigwedge}

\definecolor{gray}{rgb}{.4,.4,.4}
\definecolor{midgrey}{rgb}{0.5,0.5,0.5}
\definecolor{middarkgrey}{rgb}{0.35,0.35,0.35}
\definecolor{darkgrey}{rgb}{0.3,0.3,0.3}
\definecolor{darkred}{rgb}{0.7,0.1,0.1}
\definecolor{midblue}{rgb}{0.2,0.2,0.7}
\definecolor{darkblue}{rgb}{0.1,0.1,0.5}
\definecolor{defseagreen}{cmyk}{0.69,0,0.50,0}
\newcommand{\jnote}[1]{\medskip\noindent$\llbracket$\textcolor{darkred}{joao}: \emph{\textcolor{middarkgrey}{#1}}$\rrbracket$\medskip}

\newcommand{\jnoteF}[1]{}

\newcounter{Comment}[Comment]
\DeclareMathOperator*{\nentails}{\nvDash}
\DeclareMathOperator*{\entails}{\vDash}

\DeclareMathOperator*{\limply}{\rightarrow}

\def\topbotatom#1{\hbox{\hbox to 0pt{$#1\bot$\hss}$#1\top$}}

\setlist{nosep}
\setlistdepth{5}

\newcommand{\findaxp}{\ensuremath\mathsf{findAXp}}
\newcommand{\findcxp}{\ensuremath\mathsf{findCXp}}
\newcommand{\oneaxp}{\ensuremath\mathsf{oneAXp}}
\newcommand{\isweakaxp}{\ensuremath\mathsf{isWeakAXp}}
\newcommand{\onecxp}{\ensuremath\mathsf{oneCXp}}
\newcommand{\isweakcxp}{\ensuremath\mathsf{isWeakCXp}}
\newcommand{\prtaxp}{\ensuremath\mathsf{reportAXp}}
\newcommand{\prtcxp}{\ensuremath\mathsf{reportCXp}}
\newcommand{\isvalid}{\ensuremath\mathsf{isValid}}
\newcommand{\isconsistent}{\ensuremath\mathsf{isConsistent}}

\algnewcommand\algorithmicinput{\textbf{input:}}
\newcommand{\SAT}{\ensuremath\mathsf{SAT}}
\newcommand{\outc}{\ensuremath\mathsf{outc}}

\algnewcommand\Input{\item[\algorithmicinput]}
\newcommand{\True}{\textbf{true}}
\newcommand{\False}{\textbf{false}}
\newcommand{\NOT}{\textbf{not}}

\newcommand{\OR}{\textbf{or}}

\newcolumntype{L}[1]{>{\raggedright\let\newline\\\arraybackslash\hspace{0pt}}m{#1}}
\newcolumntype{C}[1]{>{\centering\let\newline\\\arraybackslash\hspace{0pt}}m{#1}}
\newcolumntype{R}[1]{>{\raggedleft\let\newline\\\arraybackslash\hspace{0pt}}m{#1}}

\begin{document}

\maketitle

\begin{abstract}
  Knowledge compilation (KC) languages find a growing number of
  practical uses, including in Constraint Programming (CP) and in
  Machine Learning (ML). In most applications, one natural question is 
  how to explain the decisions made by models represented by a KC
  language.
  This paper shows that for many of the best known KC languages,
  well-known classes of explanations can be computed in polynomial
  time. These classes include deterministic decomposable negation
  normal form (d-DNNF), and so any KC language that is strictly less
  succinct than d-DNNF. Furthermore, the paper also investigates the
  conditions under which polynomial time computation of explanations
  can be extended to KC languages  more succinct than d-DNNF.
\end{abstract}

\section{Introduction}
\label{sec:intro}

The growing use of machine learning (ML) models in practical
applications raises a number of concerns related with fairness,
robustness, but also
explainability~\cite{lipton-cacm18,weld-cacm19,monroe-cacm21}.
Recent years have witnessed a number of works on computing
explanations for the predictions made by ML models\footnote{There is a
  fast growing body of work on the explainability of ML
  models. Example references include~\cite{pedreschi-acmcs19,xai-bk19,muller-xai19-ch01,miller-aij19,miller-acm-xrds19,anjomshoae-aamas19,russell-fat19a,zhu-nlpcc19,klein-corr19}.}.
Approaches to computing explanations can be broadly categorized as
heuristic~\cite{guestrin-kdd16,lundberg-nips17,guestrin-aaai18}, which
offer no formal guarantees of rigor, and
non-heuristic~\cite{darwiche-ijcai18,inms-aaai19,darwiche-ecai20,marquis-kr20},
which in contrast offer strong guarantees of rigor.
% inms-nips19,msgcin-nips20,inams-aiia20
%
Non-heuristic explanation approaches can be further categorized into
compilation-based~\cite{darwiche-ijcai18,darwiche-aaai19,darwiche-ecai20}
and oracle-based~\cite{inms-aaai19,kwiatkowska-ijcai21}.

Compilation-based approaches resort to knowledge compilation (KC)
languages, often to compile the decision function associated with an
ML classifier~\cite{darwiche-ijcai18,darwiche-aaai19}.
As a result, more recent work studied KC languages from the
perspective of explainability, with the purpose of understanding the
complexity of computing
explanations~\cite{marquis-kr20,barcelo-nips20,marquis-corr21} but
also with the goal of identifying examples of queries of
interest~\cite{marquis-kr20,marquis-corr21}.
Observe that besides serving to compile the decision function of some
classifier, functions represented with KC languages can also be viewed
as classifiers.
In addition, explanations for the behavior of functions expressed in
KC languages find applications other than explaining ML models,
including explanations in constraint
programming~\cite{marquis-aij02,hooker-bk16,stuckey-cj19,guns-ecai20,guns-ijcai21}.
Furthermore, although recent
work~\cite{marquis-kr20,barcelo-nips20,marquis-corr21} analyzed the
complexity of explainability queries for different KC languages, it is
also the case that it is unknown which KC languages allow the
expressible functions to be explained efficiently, and which do not.
%
% it is not known which KC languages can be explained
% efficiently, and which cannot.
%
On the one hand,
\cite{marquis-kr20,marquis-corr21} proposes conditions not met by most
KC languages. On the other hand \cite{barcelo-nips20} studies
restricted cases of KC languages, but focusing on smallest
PI-explanations.
Also, since one key motivation for the use of KC languages is the
efficiency of reasoning, namely with respect to specific queries and
transformations~\cite{darwiche-jair02}, a natural question is whether
similar results can be obtained in the setting of explainability.

This paper studies the computational complexity of computing
PI-explanations~\cite{darwiche-ijcai18} and contrastive
explanations~\cite{miller-aij19} for classifiers represented with KC
languages.
Concretely, the paper shows that for any KC language that implements
in polynomial time the well-known queries of consistency (\tbf{CO})
and validity (\tbf{VA}), and the transformation of conditioning
(\tbf{CD}), then one PI-explanation or one contrastive explanation can
be computed in polynomial time.
This requirement is strictly less stringent than another one proposed
in earlier work~\cite{marquis-kr20}. As a result, for a large number
of KC languages, that include d-DNNF, one PI-explanation or one
contrastive explanation can be computed in polynomial time. The result
immediately generalizes to KC languages less succinct than d-DNNF,
e.g.\ OBDD, SDD, to name a few.
Moreover, for the concrete case of SDDs, the paper shows that
practical optimizations lead to clear performance gains.
Besides computing one explanation, one is often interested is obtained
multiple explanations, thus allowing a decision maker to get a better
understanding of the reasons supporting a decision. As a result, the
paper also proposes a MARCO-like~\cite{lpmms-cj16} algorithm
for the enumeration of both AXps and CXps.
Furthermore, the paper studies the computational complexity of
explaining generalizations of decision sets~\cite{leskovec-kdd16}, and
proposes conditions under which explanations can be computed in
polynomial time.
Finally, the paper studies multi-class classifiers, and again proposes
conditions for finding explanations in polynomial time.

The paper is organized as follows.
\cref{sec:prelim} introduces the definitions and notation used
throughout the paper.
\cref{sec:xpddnnf} shows that for a large class of KC languages, one
PI-explanation and one contrastive explanation can be computed in
polynomial time. Concretely, the paper shows that d-DNNF can be
explained in polynomial time, and so any less succinct language can
also be explained in polynomial time. Furthermore, the paper shows
that sentential decision diagrams (SDDs) enable practical
optimizations that yield more efficient algorithms in practice.
In addition, \cref{sec:xpddnnf} shows how to enumerate explanations
requiring one NP oracle call for each computed explanation.
\cref{sec:genxp} investigates a number of generalized classifiers,
which can be built from KC languages used as building blocks.
\cref{sec:res} assesses the computation of explanations of d-DNNF's
and SDDs in practical settings.
\cref{sec:conc} concludes the paper.

\section{Preliminaries}
\label{sec:prelim}

\subparagraph*{Classification problems \& formal explanations.}
%
%We consider a set of features $\fml{F}=\{1,\ldots,m\}$, where each
%feature takes values from $\mbb{D}=\{0,1\}$. A propositional variable
%$x_i$ is associated with each feature $i$, and the set of propositional
%variables is $X=\{x_1,\ldots,x_m\}$. Feature space is represented by
%$\mbb{F}=\mbb{D}^{m}$.

This paper considers classification problems, which are defined on a
set of features (or attributes) $\fml{F}=\{1,\ldots,m\}$ and a set of
classes $\fml{K}=\{c_1,c_2,\ldots,c_K\}$.
Each feature $i\in\fml{F}$ takes values from a domain $\mbb{D}_i$.
In general, domains can be boolean, integer or real-valued, but in
this paper we restrict $\mbb{D}_i=\{0,1\}$ and $\fml{K}=\{0,1\}$.
(In the context of KC languages, we will replace 0 by $\bot$ and 1 by
$\top$. This applies to domains and classes.)
Feature space is defined as
$\mbb{F}=\mbb{D}_1\times{\mbb{D}_2}\times\ldots\times{\mbb{D}_m}=\{0,1\}^{m}$.
The notation $\mbf{x}=(x_1,\ldots,x_m)$ denotes an arbitrary point in
feature space, where each $x_i$ is a variable taking values from
$\mbb{D}_i$. The set of variables associated with features is
$X=\{x_1,\ldots,x_m\}$.
Moreover, the notation $\mbf{v}=(v_1,\ldots,v_m)$ represents a
specific point in feature space, where each $v_i$ is a constant
representing one concrete value from $\mbb{D}_i=\{0,1\}$.
An \emph{instance} (or example) denotes a pair $(\mbf{v}, c)$, where
$\mbf{v}\in\mbb{F}$ and $c\in\fml{K}$. (We also use the term
\emph{instance} to refer to $\mbf{v}$, leaving $c$ implicit.)
An ML classifier $\mbb{C}$ is characterized by a \emph{classification
function} $\kappa$ that maps feature space $\mbb{F}$ into the set of
classes $\fml{K}$, i.e.\ $\kappa:\mbb{F}\to\fml{K}$. (It is assumed
throughout that $\kappa$ is not constant, i.e.\ there are at least two
points $\mbf{v}_1$ and $\mbf{v}_2$ in feature space, where
$\kappa(\mbf{v}_1)\not=\kappa(\mbf{v}_2)$.)

%\subparagraph*{Abductive and constrastive explanations.}
%
We now define formal explanations.
Prime implicant (PI) explanations~\cite{darwiche-ijcai18} denote a
minimal set of literals (relating a feature value $x_i$ and a constant
$v_i\in\mbb{D}_i$) %from its domain $\mbb{D}_i$
that are sufficient for the prediction\footnote{%
PI-explanations are related with abduction, and so are also referred
to as abductive explanations (AXp)~\cite{inms-aaai19}. More recently,
PI-explanations have been studied from a knowledge compilation
perspective~\cite{marquis-kr20,marquis-corr21}.}.
Formally, given $\mbf{v}=(v_1,\ldots,v_m)\in\mbb{F}$ with
$\kappa(\mbf{v})=c$, a \emph{weak} (or non-minimal) \emph{abductive
explanation} (weak AXp) is any subset $\fml{X}\subseteq\fml{F}$ such
that,
\begin{equation} \label{eq:axp}
  \forall(\mbf{x}\in\mbb{F}).
  \left[
    \bigwedge\nolimits_{i\in{\fml{X}}}(x_i=v_i)
    \right]
  \limply(\kappa(\mbf{x})=c)
\end{equation}
Any subset-minimal weak AXp is referred to as an AXp.
%%\label{page:eq:axp}
%
AXps can be viewed as answering a `Why?' question, i.e.\ why is some
prediction made given some point in feature space. A different view of
explanations is a contrastive explanation~\cite{miller-aij19}, which
answers a `Why Not?' question, i.e.\ which features can be changed to
change the prediction. A formal definition of \emph{contrastive
  explanation} (CXp) is proposed in recent work~\cite{inams-aiia20}.
Given $\mbf{v}=(v_1,\ldots,v_m)\in\mbb{F}$ with $\kappa(\mbf{v})=c$, a
\emph{weak} (or non-minimal) CXp is any subset
$\fml{Y}\subseteq\fml{F}$ such that,
\begin{equation} \label{eq:cxp}
  \exists(\mbf{x}\in\mbb{F}).\bigwedge\nolimits_{j\in\fml{F}\setminus\fml{Y}}(x_j=v_j)\land(\kappa(\mbf{x})\not=c) %\not\in\fml{Y}
\end{equation}
Any subset-minimal weak CXp is referred to as a CXp.
Building on the results of R.~Reiter in model-based
diagnosis~\cite{reiter-aij87},~\cite{inams-aiia20} proves a minimal
hitting set (MHS) duality relation between AXps and CXps,
i.e.\ AXps are MHSes of CXps and vice-versa.

\subparagraph*{Knowledge compilation map.}
Following earlier
work~\cite{darwiche-jancl01,darwiche-jair02,darwiche-jair07}, we
define negated normal form ($\nnf$), decomposable $\nnf$ ($\dnnf$),
deterministic $\dnnf$ ($\ddnnf$), decision $\dnnf$ ($\decnnf$), and
also smooth $\ddnnf$ ($\sddnnf$). 

\begin{definition}[KC languages~\cite{darwiche-jair02}]\footnote{%
    We introduce KC languages that have been studied in earlier
    works~\cite{darwiche-jair02,marquis-aaai08,darwiche-ijcai11,fargier-gkr11,marquis-ijcai13}.
    For the sake of brevity, we define only the KC languages that are
    analyzed in greater detail in the paper. For the additional KC
    languages that are mentioned in the paper, the following references give 
    standard definitions:
    OBDD~\cite{darwiche-jair02},
    PI~\cite{darwiche-jair02},
    IP~\cite{darwiche-jair02},
    renH-C~\cite{marquis-aaai08}, AFF~\cite{marquis-aaai08},
    SDD~\cite{darwiche-ijcai11}, 
    dFSD~\cite{fargier-gkr11},
    and
    EADT~\cite{marquis-ijcai13}.}
  The following KC languages are studied in the paper:
  \begin{itemize}[nosep]
    \item %\begin{definition}[$\bnnf$~\cite{darwiche-jancl01}]
      The language \emph{negated normal form} ($\nnf$) is the set of
      all directed acyclic graphs, where each leaf node is labeled
      with either $\top$, $\bot$, $x_i$ or $\neg{x_i}$, for
      $x_i\in{X}$. Each internal node is labeled with either $\land$
      (or \tn{AND}) or $\lor$ (or \tn{OR}).
      %\end{definition}
    \item %\begin{definition}[$\bdnnf$~\cite{darwiche-jancl01}]
      The language \emph{decomposable} $\nnf$ ($\dnnf$) is the set of
      all NNFs, where for every node labeled with $\land$,
      $\alpha=\alpha_1\land\cdots\land\alpha_k$, no variables are
      shared between the conjuncts $\alpha_j$.
      %\end{definition}
    \item %\begin{definition}[$\bddnnf$~\cite{darwiche-jancl01}]
      A $\ddnnf$ is a $\dnnf$, where for every node labeled with
      $\lor$, $\beta=\beta_1\lor\cdots\lor\beta_k$, each pair
      $\beta_p,\beta_q$, with $p\not=q$, is inconsistent,
      i.e.\ $\beta_p\land\beta_q\entails\bot$. 
      %\end{definition}
    %\item %\begin{definition}[$\bdecnnf$~\cite{darwiche-jair07}]
    %  A $\decnnf$ (or $\tn{decision-}\dnnf$) is a $\dnnf$, where every
    %  node labeled with $\lor$ is given by
    %  $(x_i\land\alpha)\lor(\neg{x_i}\land\beta)$.
    %  %\end{definition}
    \item %\begin{definition}[$\bsddnnf$~\cite{darwiche-jancl01}]
      An $\sddnnf$ is a $\ddnnf$, where for every node labeled with
      $\lor$, $\beta=\beta_1\lor\cdots\lor\beta_k$, each pair
      $\beta_p,\beta_q$ is defined on the same set of variables.
      %\end{definition}
  \end{itemize}
\end{definition}
The focus of this paper is $\ddnnf$, but for simplicity of algorithms,
$\sddnnf$ is often considered~\cite{darwiche-jancl01}.
Moreover, the definition of SDD is
assumed~\cite{darwiche-ijcai11,bova-aaai16} (which is briefly overview
in~\autoref{ssec:sdd}).

Throughout the paper, a term $\rho$ denotes a conjunction of
literals. A term $\rho$ is consistent ($\rho\nentails\bot$) if the
term is satisfied in at least one point in feature space.

For the purposes of this paper, we will consider exclusively the
queries \tbf{CO} and \tbf{VA}, and the transformation \tbf{CD}, which
we define next. 
Let \tbf{L} denote a subset of $\nnf$. Hence, we have the following
standard definitions~\cite{darwiche-jair02}.

\begin{definition}[Conditioning~\cite{darwiche-jair02}]\footnote{%
    We introduce the KC queries and transformations that are relevant
    for the results in the paper. There are additional queries
    (e.g.\ \tbf{CE}, \tbf{IM}, \tbf{EQ}, \tbf{SE}, \tbf{CT}, \tbf{ME})
    and transformations (e.g.\ \tbf{FO}, \tbf{SFO},
    $\pmb{\land}$\tbf{C}, $\pmb{\land}\tbf{BC}$, $\pmb{\lor}$\tbf{C},
    $\pmb{\lor}$\tbf{BC}, $\pmb{\neg}$\tbf{C}), but are omitted for
    the sake of brevity. The interested reader is referred for
    example to~\cite{darwiche-jair02}.}
  Let $\Delta$ represent a propositional formula and let $\rho$ denote
  a consistent term. The \emph{conditioning} of $\Delta$ on $\rho$,
  denoted $\Delta|_{\rho}$ is the formula obtained by replacing each
  variable $x_i$ by $\top$ (resp.~$\bot$) if $x_i$ (resp.~$\neg{x_i}$)
  is a positive (resp.~negative) literal of $\rho$.
\end{definition}

\begin{definition}[Queries \& transformations~\cite{darwiche-jair02}]
The following queries and transformations are used throughout with
respect to a KC language \tbf{L}:
  \begin{itemize}[nosep]
  \item \tbf{L} satisfies the consistency (validity) query \tbf{CO}
    (\tbf{VA}) iff there exists a polynomial-time algorithm that maps
    every formula $\Delta$ from \tbf{L} to 1 if $\Delta$ is consistent
    (valid), and to 0 otherwise.
  \item \tbf{L} satisfies the conditioning transformation \tbf{CD} iff
    there exists a polynomial-time algorithm that maps every formula
    $\Delta$ from \tbf{L} and every consistent term $\rho$ into a
    formula that is logically equivalent to $\Delta|_{\rho}$.
  \end{itemize}
\end{definition}
There are additional queries and transformations of
interest~\cite{darwiche-jair02}, but these are beyond the goals of
this paper.
$\ddnnf$ has been studied in detail from the perspective of the
knowledge compilation (KC) map~\cite{darwiche-jair02}.
Hence, it is known that $\ddnnf$ satisfies the queries \tbf{CO},
\tbf{VA}, \tbf{CE}, \tbf{IM}, \tbf{CT}, \tbf{ME}, and the
transformation \tbf{CD}.

\begin{comment}
%
\begin{definition}[Counting query, \tbf{CT}~\cite{darwiche-jair02}]
  \tbf{L} satisfies \tbf{CT} if there exists a polynomial-time
  algorithm that maps every formula $\Delta$ from \tbf{L} into a
  nonnegative integer denoting the number of models of $\Delta$
  (i.e.\ the number of assignments to the variables in $X$ for which
  $\Delta$ evaluates to $\top$.
\end{definition}
%
\end{comment}

\begin{figure}
  \begin{subfigure}[b]{\textwidth}
  %\begin{subfigure}[b]{0.5\textwidth}
    \begin{center}
      \scalebox{0.925}{\begin{tikzpicture}[-,%
    node distance={2.5cm}, thin,
    nonleaf/.style = {draw, circle},
    leafn/.style = {draw, rectangle, minimum size=0.575cm},
    level 1/.style={sibling distance=35mm},
    level 2/.style={sibling distance=25mm},
    level 3/.style={sibling distance=15mm},
  ]
  \node[nonleaf] (1) {$\land$}
  child { %[sibling distance=10mm]
    node[nonleaf] (2) {$\lor$}
    child {
      node[nonleaf] (4) {$\land$}
      child {
        node[leafn] (8) {$x_1$}
      }
      child {
        node[leafn] (9) {$x_4$}
      }
    }
    child {
      node[nonleaf] (5) {$\land$}
      child {
        node[leafn] (10) {$\neg{x_1}$}
      }
    }
  }
  child { %[sibling distance=10mm]
    node[nonleaf] (3) {$\lor$}
    child {
      node[leafn] (6) {$x_3$}
    }
    child {
      node[nonleaf] (7) {$\land$}
      child {
        node[leafn] (11) {$\neg{x_3}$}
      }
      child {
        node[leafn] (12) {$x_2$}
      }
    }
  }
  ;
  %
  %\draw[] (1) -- (2);
  %\draw[] (1) -- (3);
  %\draw[] (2) -- (4);
  %\draw[] (2) -- (5);
  %\draw[] (3) -- (6);
  %\draw[] (3) -- (7);
  %\draw[] (4) -- (8);
  %\draw[] (4) -- (9);
  \draw[] (5) -- (9);
  %\draw[] (5) -- (10);
  %\draw[] (7) -- (11);
  %\draw[] (7) -- (12);
\end{tikzpicture} }
      \caption{$\ddnnf$ $\fml{C}$ for
        $\kappa(x_1,x_2,x_3,x_4)=((x_1\land{x_4})\lor(\neg{x_1}\land{x_4}))\land(x_3\lor(\neg{x_3}\land{x_2}))$.} \label{fig:ex01a}
    \end{center}
  \end{subfigure}
  %\begin{subfigure}[b]{0.5\textwidth}
  %  \begin{center}
  %    \scalebox{0.925}{\input{./texfigs/graph01b}}
  %    \caption{Resulting $\sddnnf$} \label{fig:ex01b}
  %  \end{center}
  %\end{subfigure}
  \smallskip
  
  \begin{subfigure}{\textwidth}
    \begin{center}
      \renewcommand{\tabcolsep}{0.5em}
      \renewcommand{\arraystretch}{1.05}
      \begin{tabular}{r|cccccccccccccccc} \toprule
        $x_1$
        & 0 & 0 & 0 & 0 & 0 & 0 & 0 & 0 & 1 & 1 & 1 & 1 & 1 & 1 & 1 & 1 
        \\
        $x_2$
        & 0 & 0 & 0 & 0 & 1 & 1 & 1 & 1 & 0 & 0 & 0 & 0 & 1 & 1 & 1 & 1 
        \\
        $x_3$
        & 0 & 0 & 1 & 1 & 0 & 0 & 1 & 1 & 0 & 0 & 1 & 1 & 0 & 0 & 1 & 1 
        \\
        $x_4$
        & 0 & 1 & 0 & 1 & 0 & 1 & 0 & 1 & 0 & 1 & 0 & 1 & 0 & 1 & 0 & 1 
        \\ \midrule
        $\kappa(x_1,x_2,x_3,x_4)$
        & 0 & 0 & 0 & 1 & 0 & 1 & 0 & 1 & 0 & 0 & 0 & 1 & 0 & 1 & 0 & 1 
        \\ \bottomrule
      \end{tabular}
      \caption{Truth table for d-DNNF $\fml{C}$.
        Throughout the paper, the instance considered is
        $\mbf{v}=(0,0,0,0)$, with prediction $c=0$.}
      %, showing the number of models
      \label{fig:ex01c}
    \end{center}
  \end{subfigure}
  \caption{Running example (adapted from~\cite{stuckey-cj19}).}
  \label{fig:ex01}
\end{figure}

\begin{example} \label{ex:ex01}
  \autoref{fig:ex01} shows the running example used throughout the
  paper.
  $\fml{F}=\{1,2,3,4\}$, $X=\{x_1,x_2,x_3,x_4\}$, and
  $\kappa(x_1,x_2,x_3,x_4)=((x_1\land{x_4})\lor(\neg{x_1}\land{x_4}))\land(x_3\lor(\neg{x_3}\land{x_2}))$. Moreover,
  the paper considers the concrete instance
  $(\mbf{v},c)=((0,0,0,0),0)$.
  %
  %An example $\ddnnf$ is shown in~\autoref{fig:ex01}. We have
  %$\fml{F}=\{1,2,3,4\}$ and $X=\{x_1,x_2,x_3,x_4\}$.
  %The starting $\ddnnf$ is shown in \autoref{fig:ex01a}.
  %The transformed $\sddnnf$, obtained using the reduction proposed
  %in~\cite{darwiche-jancl01,darwiche-jair02}, is shown in
  %\autoref{fig:ex01b}. Using the polynomial-time algorithm for model
  %counting on the $\sddnnf$, proposed in~\cite{darwiche-jancl01},
  %%earlier work
  %we can  conclude that the number of models is 6. This is confirmed
  %in \autoref{fig:ex01c}.
\end{example}

\subparagraph*{Canonical KC languages.}
Some widely used KC languages are canonical, i.e.\ equivalent
functions have the same representation. Concrete examples
include\footnote{%
  The paper briefly covers examples of canonical KC languages but, for 
  the sake of brevity, does not define them. Definitions can be found
  in the references provided.} 
reduced ordered decision diagrams
(OBDDs)~\cite{bryant-comp86,darwiche-jair02},
reduced ordered multi-valued decision diagrams
MDDs~\cite{brayton-iccad90,hooker-bk16},
but also sentential decision diagrams
SDDs~\cite{darwiche-ijcai11}.
(Although we use the acronyms that are used in the literature, all
these canonical representations involve some fixed order of the
variables, and the resulting representation is reduced.)
As shown later, for the purposes of this paper, canonicity can play a
crucial role in reducing the complexity of explanation algorithms.

\subparagraph*{Related Work.}
%%~\\
%
%
PI-explanations have been studied in a growing number of
works~\cite{darwiche-ijcai18,darwiche-aaai19,inms-aaai19,inms-nips19,darwiche-ecai20,marquis-kr20,barcelo-nips20,inams-aiia20,kutyniok-jair21,kwiatkowska-ijcai21,marquis-corr21}.
Some of these earlier works studied PI-explanations for KC
languages~\cite{darwiche-ijcai18,darwiche-aaai19,darwiche-ecai20,marquis-kr20,barcelo-nips20,marquis-corr21}. %
However, results on the efficient computation of explanations for 
well-known KC languages are scarce.
For example, \cite{darwiche-ijcai18,darwiche-aaai19,darwiche-ecai20}
propose compilation algorithms (which are worst-case exponential) to
generate the PI-explanations from OBDDs.
Concretely, a classifier is compiled into an OBDD, which is then
compiled into an OBDD representing the PI-explanations of the original
classifier.
Furthermore,~\cite{marquis-kr20} proves that if a KC language satisfies
\tbf{CD}, \tbf{FO}, and \tbf{IM}, then one PI-explanation can be
computed in polynomial time. Unfortunately, a large number of KC
languages of interest do \emph{not} simultaneously satisfy \tbf{CD},
\tbf{FO}, and \tbf{IM}. This is the case for example with OBDD, SDD,
d-DNNF, among others. Moreover, ~\cite{marquis-kr20} proves that
there are polynomial time algorithms for d-DNNF for a number of
XAI-relevant queries, with the exception of DPI (deriving a prime
implicant explanation).
Finally, Barcel\'{o} et al.~\cite{barcelo-nips20} focus on smallest
PI-explanations, and prove a number of NP-hardness results.
%
%

%%\subparagraph*{Related work\hfill[Work in Progress].}
%%%
%%~\\
Knowledge compilation (KC) languages also find a growing range of
applications in constraint programming. 
Concrete examples include the compilation of constraints into
Multi-Valued Decision Diagrams~\cite{vanhoeve-cp20,stuckey-cj19} (and
their use in the context of multi-objective
optimization~\cite{hooker-bk16}, among a number of other use cases) or
d-DNNFs~\cite{stuckey-cj19}, but also for restoring consistency and
computing explanations of dynamic CSPs~\cite{marquis-aij02}, among
others.
%
%
%
%Similarly, KC languages have been used as classifiers%
%machine learning (ML) models
%~\cite{darwiche-ijcai18,darwiche-aaai19,darwiche-ecai20,marquis-kr20}.
%
Although explanations for classifiers find a growing interest in ML
and related fields, explanations of KC languages can also find a wider
range of applications, including reasoning about compiled
constraints.
Moreover, even though recent years have witnessed a growing interest
in finding explanations of machine learning (ML)
models~\cite{lipton-cacm18,pedreschi-acmcs19,weld-cacm19,monroe-cacm21},
explanations have been studied from different perspectives and in
different branches of AI at least since the
80s~\cite{shanahan-ijcai89,simari-aij02,uzcategui-aij03},
including more recently in constraint
programming~\cite{marquis-aij02,guns-ecai20,guns-ijcai21}.
The use of NP oracles for computing explanations has also been
investigated in recent years~\cite{inms-aaai19,kwiatkowska-ijcai21},
where the NP oracle can represent a CP/SMT/MILP reasoner. (With a mild
abuse of notation, when we refer to an NP oracle it is assumed that for
the accepted instances, a \emph{witness} will be returned by the oracle.)

\begin{comment}

Furthermore, recent work has also revealed a number of important
connections between knowledge compilation and explanations of ML
models~\cite{darwiche-ijcai18,darwiche-aaai19,darwiche-ecai20,marquis-kr20,marquis-corr21}. 
%

\jnote{Include Martin's arguments.}

%More importantly, a number of works in the past investigated the
%connections between computing explanations in general and constraint
%programming~\cite{shanahan-ijcai89,simari-aij02,marquis-aij02,uzcategui-aij03}.

%
More recently, there has been a renewed interest of explanations due
to the growing need to explain the predictions made by machine
learning (ML) models~\cite{}.
%
%

(...)\\
The use of oracles has also been
investigated~\cite{inms-aaai19,kwiatkowska-ijcai21}. 

%The computation of explanations has been studied in different AI
%settings including constraint programming....
%

\end{comment}

\section{Explanations for d-DNNF \& Related Languages}
\label{sec:xpddnnf}

%%\subparagraph*{Explanations vs.~queries \& transformations.}
%
As will be shown in this section, there is a tight connection between
the definitions of AXp and CXp (see~\eqref{eq:axp} and~\eqref{eq:cxp})
and the queries \tbf{VA}, \tbf{CO} and the transformation \tbf{CD}.
Indeed, for \eqref{eq:axp} and \eqref{eq:cxp}, \tbf{CD} can serve to
impose that the values of some features ($i$, represented by variable
$x_i$) are fixed to some value $v_i$.
In addition, \tbf{VA} (resp.~\tbf{CO}) is used to decide
\eqref{eq:axp}, after conditioning, when $c=1$ (resp.~$c=0$).
Similarly, \tbf{VA} (resp.~\tbf{CO}) is used to decide
\eqref{eq:cxp}, again after conditioning, when $c=1$ (resp.~$c=0$).
Thus, for languages respecting the (poly-time) queries \tbf{VA} and
\tbf{CO} and the (poly-time) transformation \tbf{CD}, one can compute
one AXp and one CXp in polynomial time.
The next sections formalize this intuition.
Furthermore, even though our focus is the d-DNNF KC language, we also
show that results in this section apply to any KC language respecting
the queries \tbf{CO}, \tbf{VA} and the transformation \tbf{CD}.

\subsection{Finding one AXp}
\label{ssec:1axp}

%%\subsection{A more direct proof}

\begin{comment}
Given $\kappa:\{0,1\}^n \rightarrow \{0,1\}$ and a vector $\mbf{v}$
with $\kappa(\mbf{v})=c$, a \emph{weak AXp} is a subset $S \subseteq
\fml{F} = \{1,\ldots,n\}$ such that
\[ \forall \mbf{x} \in \{0,1\}^n \left( \bigwedge_{i \in S} x_i=v_i \ \rightarrow \ \kappa(\mbf{x})=c \right)
\]
$S$ is an \emph{AXp} if it is a subset-minimal weak AXp.
\end{comment}

This section details an algorithm to find one AXp. %%\\
We identify any $\fml{S}\subseteq\{1,\ldots,m\}$ with its
corresponding bit-vector $\mbf{s}=(s_1,\ldots,s_m)$ where
$s_i=1\Leftrightarrow{i}\in\fml{S}$. Given
vectors $\mbf{x},\mbf{v},\mbf{s}$, we can construct the vector
$\mbf{x}^{\mbf{s},\mbf{v}}$ (in which $\mbf{s}$ is a selector between
the two vectors $\mbf{x}$ and $\mbf{v}$) such that
\begin{equation} \label{eq:trf}
  x^{\mbf{s},\mbf{v}}_i = (x_i \land \overline{s_i}) \lor (v_i \land s_i)
\end{equation}

To find an AXp, i.e.\ a subset-minimal weak AXp, \autoref{alg:oneaxp}
is used.
\begin{algorithm}[t]
  \hspace*{\algorithmicindent}
\textbf{Input}: {Classifier $\kappa$, instance $\mbf{v}$}\\
\hspace*{\algorithmicindent}
\textbf{Output}: {AXp $\fml{S}$}
\begin{algorithmic}[1]
  \Procedure{$\oneaxp$}{$\kappa,\mbf{v}$}
  \State{$\fml{S} \gets \{1,\ldots,m\}$}
  \For{$i\in\{1,\ldots,m\}$}
    \If{$\isweakaxp(\fml{S}\setminus\{i\}, \kappa(\mbf{x}^{\mbf{s},\mbf{v}})=c)$}
      \State{$\fml{S} \gets \fml{S}\setminus\{i\}$}
    \EndIf
  \EndFor  
  \State{\bfseries{return}~{$\fml{S}$}}
\EndProcedure
\end{algorithmic}

  \caption{Finding one AXp} \label{alg:oneaxp}
\end{algorithm}
%
%\begin{tabbing}
%\ \ \ \= $S$ := $\{1,\ldots,m\}$ ; \\
%\> for \= $i=1,\ldots,m$ : \\
%\> \> if $S \setminus \{i\}$ is a weak AXp of
%$\kappa(\mbf{x}^{\mbf{s},\mbf{v}})=c$ \\
%%$\kappa(\mbf{v})=c$ \\
%\> \> then $S$ := $S \setminus \{i\}$ ;
%\end{tabbing}
%
(\autoref{alg:oneaxp} is a general greedy algorithm that is well-known
and used in a wide range of settings, e.g.\ minimal unsatisfiable core
extraction in CSPs~\cite{chinneck-jc91,bakker-ijcai93}; to the best of
our knowledge, its use in finding AXps (and also CXps) of KC languages
is novel. An alternative would be to use the QuickXplain
algorithm~\cite{junker04}.)
% msbj-cav13
% ToDo: Add additional reference!!!

Considering $\mbf{s}$ and $\mbf{v}$ as constants, when $c=1$,
$\kappa(\mbf{x}^{\mbf{s},\mbf{v}})$ is valid iff $S$ is a weak AXp of
$\kappa(\mbf{v})=c$.
Furthermore, when $c=0$, $\kappa(\mbf{x}^{\mbf{s},\mbf{v}})$ is
inconsistent iff $S$ is a weak AXp of $\kappa(\mbf{v})=c$. We
therefore have the following proposition.

\begin{proposition} \label{prop:AXp2}
  For a classifier implemented with some KC language \tbf{L}, finding
  one AXp is polynomial-time provided the following three operations
  can be performed in polynomial time:
  \begin{enumerate}
  \item construction of  $\kappa(\mbf{x}^{\mbf{s},\mbf{v}})$ from
    $\kappa$, $\mbf{s}$ and $\mbf{v}$.
  \item testing validity of  $\kappa(\mbf{x}^{\mbf{s},\mbf{v}})$.
  \item testing consistency of  $\kappa(\mbf{x}^{\mbf{s},\mbf{v}})$.
  \end{enumerate}
\end{proposition}

\begin{corollary}
  Finding one AXp of a decision taken by a d-DNNF is polynomial-time.
\end{corollary}

\begin{proof}
It is sufficient to show that
d-DNNF's satisfy the conditions of Proposition~\ref{prop:AXp2}.
It is well known that testing consistency and validity d-DNNF's can be achieved in polynomial time \cite{darwiche-jair02}.
To transform a d-DNNF calculating $\kappa(\mbf{v})$ into a d-DNNF calculating $\kappa(\mbf{x}^{\mbf{s},\mbf{v}})$,
we need to replace each leaf labelled $x_i$ by a leaf labelled $(x_i \land \overline{s_i}) \lor (v_i \land s_i)$
and each leaf labelled $\overline{x_i}$ by a leaf labelled
$(\overline{x_i} \land \overline{s_i}) \lor (\overline{v_i} \land s_i)$.
Note that $\mbf{s}$ and $\mbf{v}$ are constants during this construction. Thus, we simplify
these formulas to obtain either a literal or a constant according to the different cases:
\begin{itemize}
\item $s_i=0$: \ label  $(x_i \land \overline{s_i}) \lor (v_i \land s_i)$ is $x_i$ and
label $(\overline{x_i} \land \overline{s_i}) \lor (\overline{v_i} \land s_i)$ is $\overline{x_i}$.
In other words, the label of the leaf node is unchanged.
\item $s_i=1$: \  label  $(x_i \land \overline{s_i}) \lor (v_i \land s_i)$ is the (constant) value of $v_i$ and
label $(\overline{x_i} \land \overline{s_i}) \lor (\overline{v_i} \land s_i)$ is the (constant) value of $\overline{v_i}$.
\end{itemize}
Indeed, this is just conditioning (\tbf{CD}, i.e.\ fixing a subset of
the variables $x_i$, given by the set $S$, to $v_i$) and it is well
known that \tbf{CD} is a polytime operation on d-DNNFs \cite{darwiche-jair02}.
\end{proof}

\begin{corollary}
  Finding one AXp of a decision taken by a classifier is
  polynomial-time if the classifier is given in one of the following
  languages:
  cd-PDAG~\cite{haenni-kr06},
  SDD~\cite{darwiche-ijcai11},
  OBDD~\cite{darwiche-jair02},
  PI~\cite{darwiche-jair02}, IP~\cite{darwiche-jair02},
  renH-C~\cite{marquis-aaai08}, AFF~\cite{marquis-aaai08},
  dFSD~\cite{fargier-gkr11}, and EADT~\cite{marquis-ijcai13}.
\end{corollary}

\begin{proof}
  It suffices to show that the languages listed above satisfy the conditions
  of Proposition~\ref{prop:AXp2}.  According to
  \cite{darwiche-jair02}, the queries \tbf{CO} and \tbf{VA} together with the
  transformation \tbf{CD} can all be performed in polynomial time for any of
  the languages listed above. This is exactly what we need to satisfy
  the three conditions of Proposition \ref{prop:AXp2}.
\end{proof}

\begin{figure}[t]
  \hspace*{-0.1cm}\scalebox{0.8175}{\begin{tikzpicture}[-,%
    node distance={2.5cm}, thin,
    nonleaf/.style = {draw, circle},
    leafn/.style = {draw, rectangle, minimum size=0.525cm},
    level 1/.style={sibling distance=72mm},
    level 2/.style={sibling distance=50mm},
    level 3/.style={sibling distance=32.5mm},
    level 4/.style={sibling distance=17mm},
    level 5/.style={sibling distance=9mm},
  ]
  \node[nonleaf] (1) {$\land$}
  child { %[sibling distance=10mm]
    node[nonleaf] (2) {$\lor$}
    child[sibling distance=45mm] {
      node[nonleaf] (4) {$\land$}
      child[sibling distance=30mm] {
        node[nonleaf] (8a) {$\lor$}
        child {
          node[nonleaf] (8b) {$\land$}
          child[sibling distance=5mm] {
            node[leafn] (8d) {$s_1$}
          }
          child {
            node[leafn] (8e) {$\bot$}
          }
        }
        child {
          node[nonleaf] (8c) {$\land$}
          child {
            node[leafn] (8e) {$\neg{s_1}$}
          }
          child {
            node[leafn] (8f) {$x_1$}
          }
        }
      }
      child {
        %node[leafn] (9) {$x_4$}
        node[nonleaf] (9a) {$\lor$}
        child {
          node[nonleaf] (9b) {$\land$}
          child[sibling distance=5mm] {
            node[leafn] (9d) {$s_4$}
          }
          child {
            node[leafn] (9e) {$\bot$}
          }
        }
        child {
          node[nonleaf] (9c) {$\land$}
          child {
            node[leafn] (9e) {$\neg{s_4}$}
          }
          child {
            node[leafn] (9f) {$x_4$}
          }
        }
      }
    }
    child {
      node[nonleaf] (5) {$\land$}
      child {
        %node[leafn] (10) {$\neg{x_1}$}
        node[nonleaf] (10a) {$\lor$}
        child {
          node[nonleaf] (10b) {$\land$}
          child[sibling distance=5mm] {
            node[leafn] (10d) {$s_1$}
          }
          child {
            node[leafn] (10e) {$\top$}
          }
        }
        child {
          node[nonleaf] (10c) {$\land$}
          child {
            node[leafn] (10e) {$\neg{s_1}$}
          }
          child {
            node[leafn] (10f) {$\neg{x_1}$}
          }
        }
      }
    }
  }
  child[sibling distance=82mm] {
    node[nonleaf] (3) {$\lor$}
    child {
      %node[leafn] (6) {$x_3$}
      node[nonleaf] (6a) {$\lor$}
        child[sibling distance=16mm] {
          node[nonleaf] (6b) {$\land$}
          child[sibling distance=7mm] {
            node[leafn] (6d) {$s_3$}
          }
          child[sibling distance=7mm] {
            node[leafn] (6e) {$\bot$}
          }
        }
        child[sibling distance=16mm] {
          node[nonleaf] (6c) {$\land$}
          child[sibling distance=8mm] {
            node[leafn] (6e) {$\neg{s_3}$}
          }
          child[sibling distance=8mm] {
            node[leafn] (6f) {$x_3$}
          }
        }
    }
    child[sibling distance=38.5mm] {
      node[nonleaf] (7) {$\land$}
      child {
        %node[leafn] (11) {$\neg{x_3}$}
        node[nonleaf] (11a) {$\lor$}
        child {
          node[nonleaf] (11b) {$\land$}
          child[sibling distance=5mm] {
            node[leafn] (11d) {$s_3$}
          }
          child[sibling distance=8mm] {
            node[leafn] (11e) {$\top$}
          }
        }
        child[sibling distance=15mm] {
          node[nonleaf] (11c) {$\land$}
          child {
            node[leafn] (11e) {$\neg{s_3}$}
          }
          child {
            node[leafn] (11f) {$\neg{x_3}$}
          }
        }
      }
      child {
        %node[leafn] (12) {$x_2$}
        node[nonleaf] (12a) {$\lor$}
        child { %[sibling distance=15mm]
          node[nonleaf] (12b) {$\land$}
          child[sibling distance=7mm] {
            node[leafn] (12d) {$s_2$}
          }
          child[sibling distance=7mm] {
            node[leafn] (12e) {$\bot$}
          }
        }
        child[sibling distance=15mm] {
          node[nonleaf] (12c) {$\land$}
          child {
            node[leafn] (12e) {$\neg{s_2}$}
          }
          child[sibling distance=6.5mm] {
            node[leafn] (12f) {$x_2$}
          }
        }
      }
    }
  }
  ;
  %
  %\draw[] (1) -- (2);
  %\draw[] (1) -- (3);
  %\draw[] (2) -- (4);
  %\draw[] (2) -- (5);
  %\draw[] (3) -- (6);
  %\draw[] (3) -- (7);
  %\draw[] (4) -- (8);
  %\draw[] (4) -- (9);
  \draw[] (5) -- (9a);
  %\draw[] (5) -- (10);
  %\draw[] (7) -- (11);
  %\draw[] (7) -- (12);
\end{tikzpicture} }
  \caption{Modified d-DNNF, computing
    $\kappa(\mbf{x}^{\mbf{s},\mbf{v}})$ for the instance $\mbf{v}=(0,0,0,0)$. %$\mbf{v}$ represents constant values.
    For any pick of elements to include in the weak
    AXp, $\mbf{s}$ %also
    represents constant values.} \label{fig:ex02}
\end{figure}

\begin{example}
  The operation of the algorithm is illustrated for the d-DNNF
  from~\autoref{ex:ex01}. By applying~\eqref{eq:trf}, the d-DNNF
  of~\autoref{fig:ex02} is obtained.
  The execution of the algorithm is summarized in
  \autoref{tab:ex01-axp}.
  By inspection, we can observe that the value computed by the d-DNNF
  will be 0 as long as $s_4=1$, i.e.\ as long as $4$ is part of the
  weak AXp. If removed from the weak AXp, one can find an assignment
  to $\mbf{x}$, which sets $\kappa(\mbf{x}^{\mbf{s},\mbf{v}})=1$.
  The computed AXp is $\fml{S}=\{4\}$.
\end{example}

\begin{table}[t]
  \begin{center}
    \renewcommand{\tabcolsep}{0.35em}
    \begin{tabular}{ccccc}
      \toprule
      $i$ & $\mbf{s}$   & $\kappa(\mbf{x}^{\mbf{s},\mbf{v}})$ &
      Justification & Decision \\
      \toprule
      1   & $(0,1,1,1)$ &  0 &
      $s_4=1$: left branch takes value 0, and so
      $\kappa(\mbf{x}^{\mbf{s},\mbf{v}})=0$ & Drop 1 \\
      2   & $(0,0,1,1)$ &  0 &
      $s_4=1$: left branch takes value 0, and so
      $\kappa(\mbf{x}^{\mbf{s},\mbf{v}})=0$ & Drop 2 \\
      3   & $(0,0,0,1)$ &  0 &
      $s_4=1$: left branch takes value 0, and so
      $\kappa(\mbf{x}^{\mbf{s},\mbf{v}})=0$ & Drop 3 \\
      4   & $(0,0,0,0)$ &  1 &
      Simply set $\mbf{x}=(1,1,1,1)$, and so
      $\kappa(\mbf{x}^{\mbf{s},\mbf{v}})=1$ &
      Keep 4 \\
      \bottomrule
    \end{tabular}
  \end{center}
  \caption{Example of finding one AXp} \label{tab:ex01-axp}
\end{table}

\subsection{Finding one CXp}
\label{ssec:1cxp}

To compute one CXp, \eqref{eq:cxp} is used. In this case, we identify
any $\fml{S}\subseteq\{1,\ldots,m\}$ with its corresponding bit-vector
$\mbf{s}$ where $s_i=1\Leftrightarrow{i}\in\fml{F}\setminus\fml{S}$. Moreover,
we adapt the approach used for computing one AXp, as shown in
\autoref{alg:onecxp}.
\begin{algorithm}[t]
  \hspace*{\algorithmicindent}
\textbf{Input}: {Classifier $\kappa$, instance $\mbf{v}$}\\
\hspace*{\algorithmicindent}
\textbf{Output}: {CXp $\fml{S}$}
\begin{algorithmic}[1]
  \Procedure{$\onecxp$}{$\kappa,\mbf{v}$}
  \State{$\fml{S} \gets \{1,\ldots,m\}$}
  %\If{$\tn{\NOT}\:\isweakcxp(\fml{S}, 
  %\kappa(\mbf{x}^{\mbf{s},\mbf{v}})=c)$}
  %  \State{report no CXp exists}
  %\EndIf
  \For{$i\in\{1,\ldots,m\}$}
    \If{$\isweakcxp(\fml{S}\setminus\{i\}, \kappa(\mbf{x}^{\mbf{s},\mbf{v}})=c)$}
      \State{$\fml{S} \gets \fml{S}\setminus\{i\}$}
    \EndIf
  \EndFor  
  \State{\bfseries{return}~{$\fml{S}$}}
\EndProcedure
\end{algorithmic}

  \caption{Finding one CXp} \label{alg:onecxp}
\end{algorithm}
(Observe that the main difference is the relationship between $\fml{S}$
and $\mbf{s}$, and the test for a weak CXp, that uses \eqref{eq:cxp}
with $\fml{Y}=\fml{S}$. Also, recall from~\autoref{sec:prelim} that
$\kappa$ is assumed not to be constant, and so a CXp can always be
computed.)
%
%\begin{tabbing}
%\ \ \ \= $S$ := $\{1,\ldots,m\}$ ; \\
%\> for \= $i=1,\ldots,m$ : \\
%\> \> if $S \setminus \{i\}$ is a weak CXp of $\kappa(\mbf{v})=c$ \\
%\> \> then $S$ := $S \setminus \{i\}$ ;
%\end{tabbing}
%where to test for a weak CXp, \eqref{eq:cxp} is used with
%$\fml{Y}=\fml{S}$.

\begin{proposition} \label{prop:CXp2}
  For a classifier implemented with some KC language \tbf{L}, finding
  one CXp is polynomial-time provided the operations of
  \autoref{prop:AXp2} can be performed in polynomial time.
\end{proposition}

\begin{corollary}
  Finding one CXp of a decision taken by a classifier is
  polynomial-time if the classifier is given in one of the following
  languages: d-DNNF~\cite{darwiche-jair02},
  cd-PDAG~\cite{haenni-kr06},
  SDD~\cite{darwiche-ijcai11},
  OBDD~\cite{darwiche-jair02},
  PI~\cite{darwiche-jair02}, IP~\cite{darwiche-jair02},
  renH-C~\cite{marquis-aaai08}, AFF~\cite{marquis-aaai08},
  dFSD~\cite{fargier-gkr11}, and
  EADT~\cite{marquis-ijcai13}.
\end{corollary}

\begin{example}
  The operation of the algorithm for computing one CXp is illustrated
  for the modified d-DNNF shown in~\autoref{fig:ex02} for the instance
  $(\mbf{v},c)=((0,0,0,0),0)$.
  The execution of the algorithm is summarized in
  \autoref{tab:ex01-cxp}.
  By inspection, we can observe that the value computed by the d-DNNF
  can be changed to 1 as long as $s_3=0\land{s_4}=0$, i.e.\ as long as
  $\{3,4\}$ are part of the weak CXp. If removed from the weak CXp,
  one no longer can find an assignment to $\mbf{x}$ that sets
  $\kappa(\mbf{x}^{\mbf{s},\mbf{v}})=1$.
  Thus, the computed CXp is $\fml{S}=\{3,4\}$.
\end{example}

\begin{table}[t]
  \smallskip
  \begin{center}
    \renewcommand{\tabcolsep}{0.35em}
    \begin{tabular}{ccccc}
      \toprule
      $i$ & $\mbf{s}$   & $\kappa(\mbf{x}^{\mbf{s},\mbf{v}})$ &
      Justification & Decision \\
      \toprule
      1   & $(1,0,0,0)$ &  1 &
      Pick $\mbf{x}=(0,1,1,1)$, and so
      $\kappa(\mbf{x}^{\mbf{s},\mbf{v}})=1$ &
      Drop 1 \\
      2   & $(1,1,0,0)$ &  1 &
      Pick $\mbf{x}=(0,0,1,1)$, and so
      $\kappa(\mbf{x}^{\mbf{s},\mbf{v}})=1$ &
      Drop 2 \\
      3   & $(1,1,1,0)$ &  0 &
      $s_3=1$: right branch takes value 0, and so
      $\kappa(\mbf{x}^{\mbf{s},\mbf{v}})=0$ & Keep 3 \\
      4   & $(1,1,0,1)$ &  0 &
      $s_4=1$: left branch takes value 0, and so
      $\kappa(\mbf{x}^{\mbf{s},\mbf{v}})=0$ & Keep 4 \\
      \bottomrule
    \end{tabular}
  \end{center}
  \caption{Example of finding one CXp} \label{tab:ex01-cxp}
\end{table}

\jnoteF{Use the dual of the other approach:
  \begin{enumerate}[nosep]
  \item Use the concept of weak CXp.\\
    A weak CXp is some set $\fml{Y}$ of $\fml{F}$ s.t.,
    \[
    \exists(\mbf{x}\in\mbb{F}).\bigwedge\nolimits_{j\in\fml{F}\setminus\fml{Y}}(x_j=v_j)\land(\kappa(\mbf{x})\not=c)
    %\not\in\fml{Y}
    \]
  \item While the set of features minus some feature remains a weak
    CXp, remove that feature.
  \end{enumerate}
}

\subsection{Enumerating AXps/CXps}

%This section shows that the algorithms outline in the two previous
%sections for finding one AXp and one CXp can be adapted allow for a
%seed to be specified. In turn, this enables using a ...
%
This section proposes a MARCO-like algorithm~\cite{lpmms-cj16}
for on-demand enumeration of AXps and CXps.
For that, we need to devise modified versions of
\autoref{alg:oneaxp} and \autoref{alg:onecxp}, which allow for some
initial set of features (i.e.\ a seed) to be specified.
The seed is used for computing the next AXp or CXp, and it is picked
such that repetition of explanations is disallowed.
%
%Given a picked seed, one computes either one AXp or one CXp.
%
As argued below, the algorithm's organization ensures that
computed explanations are not repeated. Moreover, since the algorithms
for computing one AXp or one CXp run in polynomial time, then the
enumeration algorithm is guaranteed to require exactly one call to an
NP oracle for each computed explanation, in addition to procedures
that run in polynomial time.

The main building blocks of the enumeration algorithm are: (1) finding
one AXp given a seed (see~\autoref{alg:findaxp}); (2) finding one CXp
given a seed (see~\autoref{alg:findcxp}); and (3) a top-level
algorithm that ensures that previously computed explanations are not
repeated (see~\autoref{alg:enum}). The top level-algorithm invokes a
SAT oracle\footnote{%
  A SAT oracle can be viewed as a modified NP oracle, that besides
  accepting/rejecting an instance (in this concrete case the formula),
  it also returns a satisfying assignment when the instance is
  satisfiable.}
to identify the seed which will determine whether a fresh AXp or CXp
will be computed in the next iteration.
%The seed ensures that only not yet computed AXps can be found.

\autoref{alg:findaxp} shows the computation of one AXp given an
initial (seed) set of features, such that any AXp that is a subset of
the given initial set of features is guaranteed not to have already
been computed.
Moreover, \autoref{alg:findcxp} shows the computation of one CXp.
As argued earlier in~\cref{ssec:1axp,ssec:1cxp}, the two algorithms
use one transformation, specifically conditioning (\tbf{CD}, see
line~\ref{alg:axp:cd}) and two queries, namely consistency
and validity (\tbf{CO}/\tbf{VA}, see line~\ref{alg:axp:cova}).
In the case of computing one AXp,
if the prediction is $\top$, we need to check validity, i.e.\ for all
(conditioned) assignments, the prediction is also $\top$. In contrast,
if the prediction is $\bot$, then we need to check that consistency
does not hold, i.e.\ for all (conditioned) assignments, the prediction
is also $\bot.$
In contrast, in the case of computing one CXp, we need to change the
tests that are executed, since we seek to change the value of the
prediction.
It should be noted that, by changing the conditioning operation,
different KC languages can be explained; this is illustrated
in~\autoref{ssec:sdd}.
\begin{algorithm}[t]
  \hspace*{\algorithmicindent}
\textbf{Input}: {Classifier $\kappa$, Seed Set $\fml{S}$, Instance
  $\mbf{v}$, Class $c$, Conditioner $\varsigma_{A}$}\\
\hspace*{\algorithmicindent}
\textbf{Output}: {AXp $\fml{S}$}
\begin{algorithmic}[1]
  \Procedure{$\findaxp$}{$\kappa,\fml{S},\mbf{v},c,\varsigma_{A}$}
  \ForAll{$i\in\fml{S}$}
    \State{\label{alg:axp:cd}$\kappa|_{\mbf{s},\mbf{v}}\gets\varsigma_{A}(\kappa,\fml{S}\setminus\{i\},\mbf{v})$}
    \If{\label{alg:axp:cova}
      $[c=\top\land\isvalid(\kappa|_{\mbf{s},\mbf{v}})]
      \:\tn{\OR}\:
      [c=\bot\land\;\!\tn{\NOT}\:\isconsistent(\kappa|_{\mbf{s},\mbf{v}})]$}
      \State{$\fml{S} \gets \fml{S}\setminus\{i\}$}
    \EndIf
  \EndFor  
  \State{\bfseries{return}~{$\fml{S}$}}
  \EndProcedure
\end{algorithmic}

  \caption{Finding one AXp given starting seed $\fml{S}$}
  \label{alg:findaxp}
\end{algorithm}
\begin{algorithm}[t]
  \hspace*{\algorithmicindent}
\textbf{Input}: {Classifier $\kappa$, Seed Set $\fml{S}$, Instance
  $\mbf{v}$, Class $c$, Conditioner $\varsigma_{C}$}\\
\hspace*{\algorithmicindent}
\textbf{Output}: {CXp $\fml{S}$}
\begin{algorithmic}[1]
  \Procedure{$\findcxp$}{$\kappa,\fml{S},\mbf{v},c,\varsigma_{C}$}
  \ForAll{$i\in\fml{S}$}
    \State{$\kappa|_{\mbf{s},\mbf{v}}\gets\varsigma_{C}(\kappa,\fml{S}\setminus\{i\},\mbf{v})$}
    \If{$[c=\top\land\;\!\tn{\NOT}\:\isvalid(\kappa|_{\mbf{s},\mbf{v}})]
      \:\tn{\OR}\:[c=\bot\land\isconsistent(\kappa|_{\mbf{s},\mbf{v}})]$}
    \State{$\fml{S} \gets \fml{S}\setminus\{i\}$}
    \EndIf
  \EndFor  
  \State{\bfseries{return}~{$\fml{S}$}}
  \EndProcedure
\end{algorithmic}

  \caption{Finding one CXp given starting seed $\fml{S}$}
  \label{alg:findcxp}
\end{algorithm}
Finally, \autoref{alg:enum} shows the proposed approach for
enumerating AXps and CXps, which adapts the basic MARCO algorithm for
enumerating minimal unsatisfiable cores~\cite{liffiton-cpaior13}.
% lpmms-cj16 !!!
% ToDo: Replace with CJ'16 reference.
From the definitions, we can see that for any $\fml{S} \subseteq \fml{F}$, either
$\fml{S}$ is a weak AXp or $\fml{F} \setminus \fml{S}$ is a weak CXp.
Every set $\fml{S}$ calculated at line 6 of \autoref{alg:enum} has the
property that it is not a superset of any previously found AXp (thanks to the
clauses added to $\fml{H}$ at line 11) and that $\fml{F} \setminus \fml{S}$
is not a superset of any previously found CXp (thanks to the clauses added
at line 15).

\begin{algorithm}[t]
  \hspace*{\algorithmicindent}
\textbf{Input}: {Feature Set $\fml{F}$, Classifier $\kappa$, Instance
  $\mbf{v}$, Class $c$, Conditioners $\varsigma_{A},\varsigma_{C}$}
%\hspace*{\algorithmicindent}
%\textbf{Input}: {XpG: $\fml{D}=(G_{\fml{D}},S,\upsilon,\alpha_{V},\alpha_{E})$}
%\\
%
\begin{algorithmic}[1]
  \Procedure{$\mathsf{Enumerate}$}{$\fml{F}, \kappa,\mbf{v},c,\varsigma_{A},\varsigma_{C}$}
  \State{%
    \label{alg:exp:ln01}$\fml{H}\gets\emptyset$}%
  \Comment{$\fml{H}$ defined on set $P=\{p_1,\ldots,p_m\}$}
  \Repeat\label{alg:exp:ln02}%
  \State{%
    \label{alg:exp:ln03} $(\outc,\mbf{p})\gets\SAT(\fml{H})$}
  \If{\label{alg:exp:ln04} $\outc=\True$}
  \State{\label{alg:exp:ln05}$\fml{S}\gets\{i\in\fml{F}\,|\,p_i=1\}$}%
  %\State{\label{alg:exp:ln06}${C}\gets\{s_i\in{S}\,|\,r_i=0\}$}%
  %
  \State{\label{alg:exp:ln06}$\kappa|_{\mbf{s},\mbf{v}}\gets\varsigma_{A}(\kappa,\fml{S},\mbf{v})$}
  \If{\label{alg:exp:ln07}
    $[c=\top\land\isvalid(\kappa|_{\mbf{s},\mbf{v}})]
    \:\tn{\OR}\:
         [c=\bot\land\;\!\tn{\NOT}\:\isconsistent(\kappa|_{\mbf{s},\mbf{v}})]$}
  \State{\label{alg:exp:ln08}$X\gets\findaxp(\kappa,\fml{S},\mbf{v},c,\varsigma_{A})$}
  \State{\label{alg:exp:ln09}$\prtaxp(X)$}
  \State{\label{alg:exp:ln10}$\fml{H}\gets\fml{H}\cup\{(\lor_{i\in{X}}\neg{p_i})\}$}
  \Else\label{alg:exp:ln11}
  %%\State{\label{alg:exp:ln12}$\kappa|_{\mbf{s},\mbf{v}}\gets\varsigma_{C}(\kappa,\fml{S},\mbf{v})$}  %No need: \fml{F} \setminus 
  %%\State{\label{alg:exp:ln13}$X\gets\findcxp(\kappa|_{\mbf{s},\mbf{v}},\fml{F}\setminus\fml{S},\mbf{v},c,\varsigma_{C})$}
  \State{\label{alg:exp:ln13}$X\gets\findcxp(\kappa,\fml{F}\setminus\fml{S},\mbf{v},c,\varsigma_{C})$}
  \State{\label{alg:exp:ln14}$\prtcxp(X)$}
  \State{\label{alg:exp:ln15}$\fml{H}\gets\fml{H}\cup\{(\lor_{i\in{X}}{p_i})\}$}
  \EndIf
  \EndIf
  \Until{\label{alg:exp:ln16}$\outc=\False$}
  \EndProcedure
\end{algorithmic}

  \caption{Enumeration algorithm} \label{alg:enum}
\end{algorithm}

\begin{example}
  \autoref{tab:ex01-enum} summarizes the main steps of enumerating the
  AXps and CXps of the running example (see~\autoref{fig:ex01}).
  It is easy to confirm that after four explanations are computed,
  $\fml{H}$ becomes inconsistent, and so the algorithm terminates.
  Also, one can confirm the hitting set duality between AXps and
  CXps~\cite{inams-aiia20}.
\end{example}

\begin{table}[t]
  \begin{center}
    \renewcommand{\tabcolsep}{0.41225em}
    \renewcommand{\arraystretch}{1.05}
    \begin{tabular}{cccC{1.25cm}ccccc} \toprule
      $\fml{H}$ & $\SAT(\fml{H})$ & $\mbf{p}$ & AXp(1), CXp(0)? & $\fml{S}$ & AXp & CXp & Block
      \\ \toprule
      $\emptyset$ & 1 & $(1,1,1,1)$ & 1 & $\{1,2,3,4\}$ & $\{4\}$ & --- &
      $b_1=(\neg{p_4})$
      \\
      $\{b_1\}$   & 1 & $(1,1,1,0)$ & 1 & $\{1,2,3\}$ & $\{2,3\}$ & --- &
      $b_2=(\neg{p_2}\lor\neg{p_3})$
      \\
      $\{b_1,b_2\}$ & 1 & $(1,0,1,0)$ & 0 & $\{1,3\}$ & --- & $\{2,4\}$ &
      $b_3=({p_2}\lor{p_4})$
      \\
      $\{b_1,b_2,b_3\}$ & 1 & $(1,1,0,0)$ & 0 & $\{1,2\}$ & --- &
      $\{3,4\}$ & $b_4=({p_3}\lor{p_4})$
      \\
      $\{b_1,b_2,b_3,b_4\}$ & 0 & --- & --- & --- & --- & --- & ---
      \\
      \bottomrule
    \end{tabular}
  \end{center}
  \caption{Example of AXp/CXp enumeration, using~\autoref{alg:enum}} \label{tab:ex01-enum}
\end{table}

%% Include here the subsection on SDDs.
\subsection{Explanations for SDDs}
\label{ssec:sdd}

As a subset of the $\ddnnf$ language, SDDs represent a well-known
KC language~\cite{darwiche-ijcai11,darwiche-aaai15,bova-aaai16}.
SDDs are based on a strongly deterministic
decomposition~\cite{darwiche-ijcai11}, which is used to decompose a
Boolean function into the form: $(p_1 \land s_1) \lor \dots \lor (p_n
\land s_n)$, where each $p_i$ is called a \textit{prime} and each
$s_i$ is called a \textit{sub} (both primes and subs are
sub-functions).
Furthermore, the process of decomposition
is governed by a variable tree (\textit{vtree}) which stipulates the variable order
\cite{darwiche-ijcai11}.
\autoref{fig:sdd_example} shows the SDD representation
of decision function $\kappa$ in \autoref{fig:ex01a}
and its {\it vtree} in \autoref{fig:sdd_vtree}.
%
\begin{comment}
In \autoref{fig:sdd_representation},
each circle node with outgoing edges is called a \textit{decision node} while
each paired-boxes node is called an \textit{element}.
The left (resp.\ right) box represents the \textit{prime} (resp.\ \textit{sub}). 
A box may either contain a terminal SDD
(i.e.\ $\top$, $\bot$ or a literal) or a link to a \textit{decision
  node}.
%
The \textit{vtree} shown in \autoref{fig:sdd_vtree} is a binary tree,
whose leaves are in a one-to-one correspondence with the domain
variables of the given Boolean function. Moreover, each SDD node
{\it respects} a unique (leaf or non-leaf) node of the {\it vtree}.
%
For example, the SDD root  in \autoref{fig:sdd_representation} respects the
root of the vtree in \autoref{fig:sdd_vtree}.
\end{comment}
%

In order to exploit \autoref{alg:findaxp}, \ref{alg:findcxp}
and \ref{alg:enum} to explain SDD classifiers,
we need to implement: (i) $\isconsistent$, (ii) $\isvalid$, and
(iii) the conditioning of decision function $\kappa$
w.r.t. $\mbf{s}$ and $\mbf{v}$ (i.e.\ $\kappa|_{\mbf{s}, \mbf{v}}$).
To compute $\kappa|_{\mbf{s}, \mbf{v}}$,
we check  each $s_i$
if ($s_i = 1$) and  we compute $\kappa|_{x_i = v_i}$
($\kappa|_{x_i}$ if $v_i = 1$, otherwise $\kappa|_{\neg x_i}$).
As SDDs satisfy $\tbf{CO}$, $\tbf{VA}$ and $\tbf{CD}$~\cite{darwiche-aaai15},
the tractability of $\isconsistent$, $\isvalid$, and  $\kappa|_{\mbf{s}, \mbf{v}}$
is guaranteed.

Next, let us consider again the running example of \autoref{fig:ex01}
and the instance $\mbf{v} = (0, 0, 0, 0)$ (such that $\kappa(\mbf{v}) = 0$).
\autoref{fig:axp_sdd} illustrates the process of computing one AXp for
$\kappa(\mbf{v})$, which corresponds to the overall flow shown in
\autoref{tab:ex01-axp}.
(Note that the computation of a CXp is similar.)
As SDDs in Figures \ref{fig:axp_sdd01}, \ref{fig:axp_sdd02} and
\ref{fig:axp_sdd03} are inconsistent, features 1, 2 and 3 are not necessary 
for preserving the prediction $\kappa(\mbf{v}) = 0$, 
that is they can be removed from $\fml{S}$.
%
%(SDDs in these sub-figures
%are presented in intermediate form for better illustrating the process.)
%
Instead, for SDD in \autoref{fig:axp_sdd04}, there exists a
point $\mbf{x}$ that can be classified as $\top$, so feature 4
cannot be removed from $\fml{S}$. Thus, we derive an AXp $\fml{S} = \{4\}$.

%%%%%%%%%%%%%%%%%%%% Example %%%%%%%%%%%%%%%%%%%%%%%
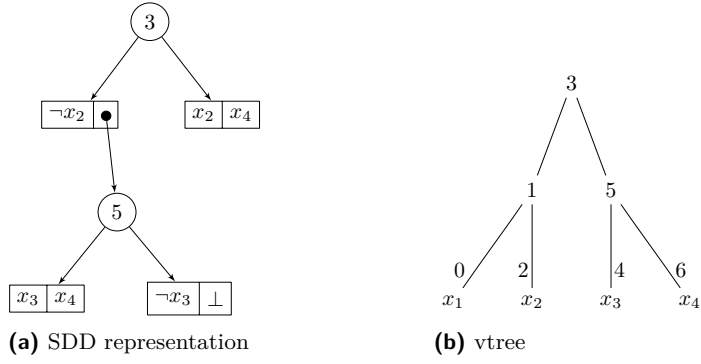
\begin{figure}[t!]
	\centering
	\begin{subfigure}[b]{0.4\textwidth}
		%\centering
		\scalebox{0.8175}{\begin{tikzpicture}[>=latex',line join=bevel,]
	node distance={2.5cm}, thin,
	\node (n15) at (55.5bp,45.263bp) [draw,fill=white,circle] {5};
	\node (n16) at (70.5bp,132.79bp) [draw,fill=white,circle] {3};
	\node (n15e0) at (23.5bp,5.5bp) [draw,fill=white,rectangle split,
	rectangle split horizontal, rectangle split parts=2] {$x_3$\nodepart{two}$x_4$};
  
	\node (n15e1) at (88.5bp,5.5bp) [draw,fill=white,rectangle split,
	rectangle split horizontal, rectangle split parts=2] {$\neg x_3$\nodepart{two}$\bot$};
  
	\node (n16e0) at (38.5bp,90.03bp) [draw,fill=white,rectangle split,
	rectangle split horizontal, rectangle split parts=2] {$\neg x_2$\nodepart{two}};
  
	\node (n16e1) at (103.5bp,90.03bp) [draw,fill=white,rectangle split,
	  rectangle split horizontal, rectangle split parts=2] {$x_2$\nodepart{two}$x_4$};
  
	\draw [->] (n15) -- (n15e0);
	\draw [->] (n15) -- (n15e1);
	\draw [->] (n16) -- (n16e0);
	\draw [->] (n16) -- (n16e1);
	%\draw [*->] (n16e0.two) -- (n15); % 'two south'
	\draw [*->] (n16e0.two)+(0.2em,0.4em) -- (n15); % 'two south'
\end{tikzpicture}}
		\caption{SDD representation}
		\label{fig:sdd_representation}
	\end{subfigure}%
	\begin{subfigure}[b]{0.4\textwidth}
		%\centering
		\scalebox{0.8175}{\begin{tikzpicture}[>=latex',line join=bevel,]
\node (n3) at (63.0bp,109.5bp) [draw,draw=none] {3};
  \node (n1) at (45.0bp,60.5bp) [draw,draw=none] {1};
  \node (n5) at (81.0bp,60.5bp) [draw,draw=none] {5};
  \node (n0) at (9.0bp,9.0bp) [draw,draw=none] {$x_1$};
  \node (n2) at (45.0bp,9.0bp) [draw,draw=none] {$x_2$};
  \node (n4) at (81.0bp,9.0bp) [draw,draw=none] {$x_3$};
  \node (n6) at (117.0bp,9.0bp) [draw,draw=none] {$x_4$};
  \draw [] (n3) ..controls (57.301bp,93.618bp) and (50.546bp,75.981bp)  .. (n1);
  \draw [] (n3) ..controls (68.699bp,93.618bp) and (75.454bp,75.981bp)  .. (n5);
  \draw [] (n1) ..controls (34.517bp,45.086bp) and (22.314bp,28.307bp)  .. (n0);
  \definecolor{strokecol}{rgb}{0.0,0.0,0.0};
  \pgfsetstrokecolor{strokecol}
  \draw (11.99bp,23.736bp) node {0};
  \draw [] (n1) ..controls (45.0bp,45.086bp) and (45.0bp,28.307bp)  .. (n2);
  %\draw (42.0bp,23.736bp) node {2};
  \draw (41.0bp,23.736bp) node {2};
  \draw [] (n5) ..controls (81.0bp,45.086bp) and (81.0bp,28.307bp)  .. (n4);
  %\draw (78.0bp,23.736bp) node {4};
  \draw (85.0bp,23.736bp) node {4};
  \draw [] (n5) ..controls (91.483bp,45.086bp) and (103.69bp,28.307bp)  .. (n6);
  %\draw (108.01bp,23.736bp) node {6};
  \draw (112.91bp,23.736bp) node {6};
\end{tikzpicture}}
		\caption{vtree}
		\label{fig:sdd_vtree}
	\end{subfigure}%
	\caption{SDD for $\kappa(x_1,x_2,x_3,x_4)=((x_1\land{x_4})\lor(\neg{x_1}\land{x_4}))\land(x_3\lor(\neg{x_3}\land{x_2}))$,
	given a vtree.
  	Each circle node with outgoing edges is  a \textit{decision node} while
	each paired-boxes node is an \textit{element}.
	The left (resp.\ right) box represents the \textit{prime} (resp.\ \textit{sub}). 
	A box  either contains a terminal SDD
	(i.e.\ $\top$, $\bot$ or a literal) or a link to a {\it decision node}.
	The shown {\it vtree} in (\ref{fig:sdd_vtree}) is a binary tree,
	whose leaves are in a one-to-one correspondence with the domain
	variables of $\kappa(x_1,x_2,x_3,x_4)$. 
	Moreover, each SDD node
	{\it respects} a unique (leaf or non-leaf) node of the {\it vtree},
	e.g.\  the SDD root of (\ref{fig:sdd_representation}) respects the
	vtree root of (\ref{fig:sdd_vtree}).	
	}
	\label{fig:sdd_example}
\end{figure}

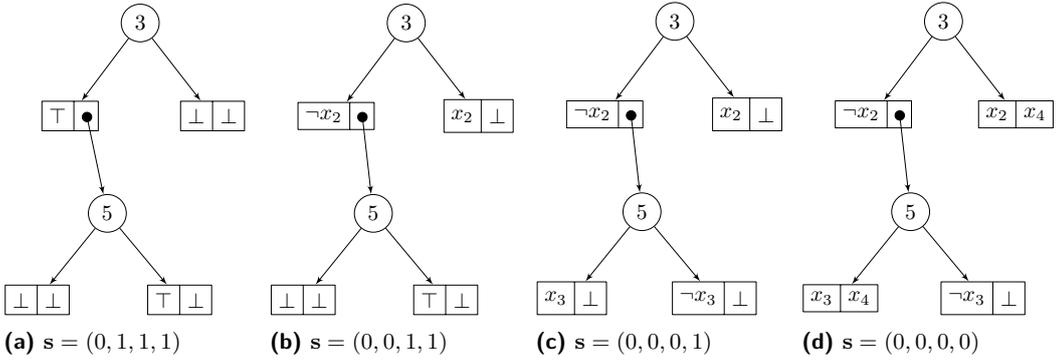
\begin{figure}[t!]
	\begin{subfigure}[b]{0.25\textwidth}
		\scalebox{0.8175}{\begin{tikzpicture}[>=latex',line join=bevel,]
	node distance={2.5cm}, thin,
	\node (n15) at (55.5bp,45.263bp) [draw,fill=white,circle] {5};
	\node (n16) at (70.5bp,132.79bp) [draw,fill=white,circle] {3};
	\node (n15e0) at (23.5bp,5.5bp) [draw,fill=white,rectangle split,
	rectangle split horizontal, rectangle split parts=2] {$\bot$\nodepart{two}$\bot$};
  
	\node (n15e1) at (88.5bp,5.5bp) [draw,fill=white,rectangle split,
	rectangle split horizontal, rectangle split parts=2] {$\top$\nodepart{two}$\bot$};
  
	\node (n16e0) at (38.5bp,90.03bp) [draw,fill=white,rectangle split,
	rectangle split horizontal, rectangle split parts=2] {$\top$\nodepart{two}};
  
	\node (n16e1) at (103.5bp,90.03bp) [draw,fill=white,rectangle split,
	rectangle split horizontal, rectangle split parts=2] {$\bot$\nodepart{two}$\bot$};
  
	\draw [->] (n15) -- (n15e0);
	\draw [->] (n15) -- (n15e1);
	\draw [->] (n16) -- (n16e0);
	\draw [->] (n16) -- (n16e1);
	\draw [*->] (n16e0.two)+(0.2em,0.4em) -- (n15);
\end{tikzpicture}}
		\caption{$\mbf{s} = (0,1,1,1)$}
		\label{fig:axp_sdd01}
	\end{subfigure}%
	\begin{subfigure}[b]{0.25\textwidth}
		\scalebox{0.8175}{\begin{tikzpicture}[>=latex',line join=bevel,]
	node distance={2.5cm}, thin,
	\node (n15) at (55.5bp,45.263bp) [draw,fill=white,circle] {5};
	\node (n16) at (70.5bp,132.79bp) [draw,fill=white,circle] {3};
	\node (n15e0) at (23.5bp,5.5bp) [draw,fill=white,rectangle split,
	rectangle split horizontal, rectangle split parts=2] {$\bot$\nodepart{two}$\bot$};
  
	\node (n15e1) at (88.5bp,5.5bp) [draw,fill=white,rectangle split,
	rectangle split horizontal, rectangle split parts=2] {$\top$\nodepart{two}$\bot$};
  
	\node (n16e0) at (38.5bp,90.03bp) [draw,fill=white,rectangle split,
	rectangle split horizontal, rectangle split parts=2] {$\neg x_2$\nodepart{two}};
  
	\node (n16e1) at (103.5bp,90.03bp) [draw,fill=white,rectangle split,
	rectangle split horizontal, rectangle split parts=2] {$x_2$\nodepart{two}$\bot$};
  
	\draw [->] (n15) -- (n15e0);
	\draw [->] (n15) -- (n15e1);
	\draw [->] (n16) -- (n16e0);
	\draw [->] (n16) -- (n16e1);
	\draw [*->] (n16e0.two)+(0.2em,0.4em) -- (n15);
\end{tikzpicture}}
		\caption{$\mbf{s} = (0,0,1,1)$}
		\label{fig:axp_sdd02}
	\end{subfigure}%
	\begin{subfigure}[b]{0.25\textwidth}
		\scalebox{0.8175}{\begin{tikzpicture}[>=latex',line join=bevel,]
	node distance={2.5cm}, thin,
	\node (n15) at (55.5bp,45.263bp) [draw,fill=white,circle] {5};
	\node (n16) at (70.5bp,132.79bp) [draw,fill=white,circle] {3};
	\node (n15e0) at (23.5bp,5.5bp) [draw,fill=white,rectangle split,
	rectangle split horizontal, rectangle split parts=2] {$x_3$\nodepart{two}$\bot$};
  
	\node (n15e1) at (88.5bp,5.5bp) [draw,fill=white,rectangle split,
	rectangle split horizontal, rectangle split parts=2] {$\neg x_3$\nodepart{two}$\bot$};
  
	\node (n16e0) at (38.5bp,90.03bp) [draw,fill=white,rectangle split,
	rectangle split horizontal, rectangle split parts=2] {$\neg x_2$\nodepart{two}};
  
	\node (n16e1) at (103.5bp,90.03bp) [draw,fill=white,rectangle split,
	rectangle split horizontal, rectangle split parts=2] {$x_2$\nodepart{two}$\bot$};
  
	\draw [->] (n15) -- (n15e0);
	\draw [->] (n15) -- (n15e1);
	\draw [->] (n16) -- (n16e0);
	\draw [->] (n16) -- (n16e1);
	\draw [*->] (n16e0.two)+(0.2em,0.4em) -- (n15);
\end{tikzpicture}}
		\caption{$\mbf{s} = (0,0,0,1)$}
		\label{fig:axp_sdd03}
	\end{subfigure}%
	\begin{subfigure}[b]{0.25\textwidth}
		\scalebox{0.8175}{\begin{tikzpicture}[>=latex',line join=bevel,]
	node distance={2.5cm}, thin,
	\node (n15) at (55.5bp,45.263bp) [draw,fill=white,circle] {5};
	\node (n16) at (70.5bp,132.79bp) [draw,fill=white,circle] {3};
	\node (n15e0) at (23.5bp,5.5bp) [draw,fill=white,rectangle split,
	rectangle split horizontal, rectangle split parts=2] {$x_3$\nodepart{two}$x_4$};
  
	\node (n15e1) at (88.5bp,5.5bp) [draw,fill=white,rectangle split,
	rectangle split horizontal, rectangle split parts=2] {$\neg x_3$\nodepart{two}$\bot$};
  
	\node (n16e0) at (38.5bp,90.03bp) [draw,fill=white,rectangle split,
	rectangle split horizontal, rectangle split parts=2] {$\neg x_2$\nodepart{two}};
  
	\node (n16e1) at (103.5bp,90.03bp) [draw,fill=white,rectangle split,
	rectangle split horizontal, rectangle split parts=2] {$x_2$\nodepart{two}$x_4$};
  
	\draw [->] (n15) -- (n15e0);
	\draw [->] (n15) -- (n15e1);
	\draw [->] (n16) -- (n16e0);
	\draw [->] (n16) -- (n16e1);
	\draw [*->] (n16e0.two)+(0.2em,0.4em) -- (n15);
\end{tikzpicture}}
		\caption{$\mbf{s} = (0,0,0,0)$}
		\label{fig:axp_sdd04}
	\end{subfigure}%
\caption{Example of computing one AXp for $\kappa(\mbf{v} = (0, 0, 0, 0)) = 0$.
		Each sub-figure represents an SDD $\kappa|_{\mbf{s}, \mbf{v}}$.
		(Note that, the SDDs are presented in intermediate form for better illustrating 
		the procedure of calculating the explanation.)
}
\label{fig:axp_sdd}.		
\end{figure}

\begin{comment}
\begin{figure}
	\begin{subfigure}[b]{0.25\textwidth}
		\scalebox{0.8175}{\input{texfigs/cxp_sdd01}}
		\caption{$\mbf{s} = (1,0,0,0)$}
		\label{fig:gull}
	\end{subfigure}%
	\begin{subfigure}[b]{0.25\textwidth}
		\scalebox{0.8175}{\input{texfigs/cxp_sdd02}}
		\caption{$\mbf{s} = (1,1,0,0)$}
		\label{fig:gull}
	\end{subfigure}%
	\begin{subfigure}[b]{0.25\textwidth}
		\scalebox{0.8175}{\input{texfigs/cxp_sdd03}}
		\caption{$\mbf{s} = (1,1,1,0)$}
		\label{fig:gull}
	\end{subfigure}%
	\begin{subfigure}[b]{0.25\textwidth}
		\scalebox{0.8175}{\input{texfigs/cxp_sdd04}}
		\caption{$\mbf{s} = (1,1,0,1)$}
		\label{fig:gull}
	\end{subfigure}%
\caption{Computing CXp}\label{fig:animals}
\end{figure}
\end{comment}
%%%%%%%%%%%%%%%%%%%% Example %%%%%%%%%%%%%%%%%%%%%%%

%%\jnote{\emph{\textbf{CLEAN UP TEXT \& FIGURES!}}}

\section{Generalizations} \label{sec:genxp}

\subsection{Explanations for Generalized Decision Functions}

We consider the setting of multi-class classification, with 
$\fml{K}=\{c_1,\ldots,c_K\}$, where each class $c_j$ is associated
with a total function $\kappa_j:\mbb{F}\to\{0,1\}$, such that the
class $c_j$ is picked iff $\kappa_j(\mbf{v})=1$.
%
%Furthermore, each class is represented by a sentence expressed in some
%KR language, such that we pick class $c_j$ 
%
For example, \emph{decision sets}~\cite{leskovec-kdd16} represent one
such example of multi-class classification, where each function
$\kappa_j$ is represented by a DNF, and a \emph{default rule} is used
to pick some class for the points $\mbf{v}$ in feature space for which
all $\kappa_j(\mbf{v})=0$.
Moreover, decision sets may exhibit
\emph{overlap}~\cite{ipnms-ijcar18}, i.e.\ the existence of points
$\mbf{v}$ in feature space such that there exist $j_1\not=j_2$
and $\kappa_{j_1}(\mbf{v})=\kappa_{j_2}(\mbf{v})=1$. In practice, the
existence of overlap can be addressed by randomly picking one of the
classes for which $\kappa_{j}(\mbf{v})=1$. Alternatively, the learning
of DSes can ensure that overlap is non-existing~\cite{ipnms-ijcar18}.

This section considers generalized versions of DSes, by removing the
restriction that each class is computed with a DNF.
Hence, a \emph{generalized decision function} (GDF) is such that each
function $\kappa_j$ is allowed to be an \emph{arbitrary} boolean
function.
Furthermore, the following two properties of GDFs are considered:
\begin{definition}[Binding GDF] \label{def:bgdf}
  A GDF is \emph{binding} if,
  \begin{equation}
    \forall(\mbf{x}\in\mbb{F}).\biglor_{1\le{j}\le{K}}\kappa_{j}(\mbf{x})
  \end{equation}
\end{definition}
(Thus, a binding GDF requires no default rule, since for any point
$\mbf{x}$ in feature space, there is at least one $\kappa_j$ such that
$\kappa_j(\mbf{x})$ holds.)

\begin{definition}[Non-overlapping GDF] \label{def:nogdf}
  A GDF is \emph{non-overlapping} if,
  \begin{equation}
    %\begin{array}{lcr}
    \forall(\mbf{x}\in\mbb{F}).
    \bigland_{\substack{1\le{j_1},{j_2}\le{K}\\{j_1}\not={j_2}}}(\neg\kappa_{j_1}(\mbf{x})\lor\neg\kappa_{j_2}(\mbf{x}))
    %\kappa_{j_1}(\mbf{x})\limply\bigland_{j_2\not=j_1}\neg\kappa_{j_2}(\mbf{x})
    %& & 1\le{j_1},{j_2}\le{m}\\
    %\end{array}
  \end{equation}
\end{definition}
%(Thus, a non-overlapping GDF effectively computes a multi-class
%classification function, which is total if the GDF is binding.)
(Thus, a binding, non-overlapping GDF computes a total multi-class
classification function.)

%In addition, for a binding non-overlapping GDF, a default rule is
%unnecessary.
%
Furthermore, we can establish conditions for a GDF to be binding and
non-overlapping:

\begin{proposition} \label{prop:gdfcond}
  A GDF is binding and non-overlapping iff the following formula is
  inconsistent:
  \begin{equation}
    \exists(\mbf{x}\in\mbb{F}).
    \kappa_{1}(\mbf{x})+\ldots+\kappa_{K}(\mbf{x})\not=1
  \end{equation}
\end{proposition}

\begin{proof}
  Given \autoref{def:bgdf} and \autoref{def:nogdf},
  \begin{enumerate}[nosep]
  \item Clearly, there exists a point $\mbf{v}\in\mbb{F}$ such that
    $\kappa_{1}(\mbf{v})+\ldots+\kappa_{K}(\mbf{v})=0$ iff the GDF
    is non-binding;
  \item Clearly, there exists $\mbf{v}\in\mbb{F}$ such that
    $\kappa_{1}(\mbf{v})+\ldots+\kappa_{K}(\mbf{v})\ge2$ iff the GDF
    is overlapping.
  \end{enumerate}
  Thus, the result follows.
  \qedhere
\end{proof}

\begin{remark}
  For a GDF where each function is represented by a boolean circuit,
  %one can use \autoref{propo:gdfcond} to prove that 
  deciding whether a GDF is binding and non-overlapping is in coNP.
  In practice, checking whether a GDF is binding and non-overlapping
  can be decided with a call to an NP oracle.
\end{remark}

%The set of functions $\kappa_j$, one for each class $c_j\in\fml{K}$,
%is referred to as a \emph{family of classifiers}.
%
%% \begin{definition}[Exclusive]
%%   A family of classifiers $\kappa_j:\mbb{F}\to\{0,1\}$, with
%%   $j=1,\ldots,K$ is \emph{exclusive} if,
%%   \begin{equation}
%%     \forall(\mbf{x}\in\mbb{F}).
%%     \kappa_{1}(\mbf{x})+\ldots+\kappa_{K}(\mbf{x})=1 %\\
%%   \end{equation}
%%   %
%%   %for each pair $1\le{p},{q}\le{K}$:
%%   %\begin{equation}
%%   %  \forall(\mbf{x}\in\mbb{F}).
%%   %  \neg\kappa_{p}(\mbf{x})\lor\neg\kappa_{q}(\mbf{x})
%%   %\end{equation}
%%   %In addition, it holds that:
%%   %\begin{equation}
%%   %  \forall(\mbf{x}\in\mbb{F}).
%%   %  \kappa_{1}(\mbf{x})\lor\ldots\lor\kappa_{K}(\mbf{x}) %\\
%%   %\end{equation}
%% \end{definition}
%
%An exclusive family of classifiers can be viewed as computing a total
%function $\kappa:\mbb{F}\to\fml{K}$.

\begin{proposition}
  For a binding and non-overlapping GDF, such that each classification
  function is represented by a sentence of a KC language satisfying
  the query $\tbf{CO}$ and the transformation $\tbf{CD}$, then one AXp
  or one CXp can be computed in polynomial time. Furthermore,
  enumeration of AXps/CXps can be achieved with one call to an NP
  oracle per computed explanation.
\end{proposition}

\begin{proof}[Proof sketch]
  For computing one AXp of class $c_p$, one can iteratively check
  consistency of the remaining of literals on the other functions
  $q\not=p$. Conditioning is used to reflect, in the classifiers, the
  choices made, i.e.\ which literals are included or not in the AXp.
  For a CXp a similar approach can be used. For enumeration, we can
  once again exploit a MARCO-like algorithm.
  \qedhere
\end{proof}

\begin{corollary}
  For a binding non-overlapping GDF, where each $\kappa_j$ is
  represented by a $\dnnf$, one AXp and one CXp can be computed in
  polynomial time.
  Furthermore,
  enumeration of AXps/CXps can be achieved with one call to an NP
  oracle per computed explanation. 
\end{corollary}

Thus, for GDFs that are both binding and non-overlapping, even if each
function is represented by the fairly succinct $\dnnf$, one can still
compute AXps and CXps efficiently. Furthermore, a
MARCO-like~\cite{lpmms-cj16} can be used for enumerating AXps
and CXps.

The results above can be generalized to the case of multi-valued
classification, where binarization (one-hot-encoding) can serve for
representing multi-valued (non-continuous) features. Alternative
approaches have been investigated in recent work~\cite{marquis-kr20}.

%\begin{remark}[Multi-Valued Classification]
%  We can consider binarization (one-hot-encoding) for representing
%  multi-valued (non-continuous) features, and then exploit the
%  approach outlined in earlier work~\cite{hiims-corr21} for handling
%  multi-valued features, i.e.\ each group of binary features encoding
%  a given non-binary feature is analyzed simultaneously.
%\end{remark}

\subsection{Total Congruent Classifiers}

We can build on the conditions for GDFs to devise relaxed conditions
for poly-time explainability.

\begin{definition}[Total Classifier]
  A classification function is total if for any point
  $\mbf{v}\in\mbb{F}$, there is a prediction $\kappa(\mbf{v})=c$, with
  $c\in\fml{K}$.
\end{definition}
%\[
%\forall(\mbf{x}\in\mbb{F}).\sum_{c\in\fml{K}}(\kappa(\mbf{x}
%\]

\begin{definition}[Congruent Classifier]
  A classifier is \emph{congruent} if the computational complexity of
  deciding the consistency of $\kappa(\mbf{x})=c$ is the same for any
  $c\in\fml{K}$.
\end{definition}

Similarly, we can define a total congruent KR language. For a total
congruent KR language, the query $\tbf{CO}$ is satisfied iff 
deciding $\kappa(\mbf{v})=c$ is in polynomial time for any
$c\in\fml{K}$.
Given the above, the same argument used for GDFs, can be used to prove
that,

\begin{proposition}
  For a total congruent KR language, which satisfies the operations of
  \tbf{CO} and \tbf{CD}, one AXp and one CXp can be computed in
  polynomial time.
\end{proposition}

%\begin{proof}[Proof sketch]
%  The argument mimics the one used for GDF.
%  \qedhere
%\end{proof}

%\begin{corollary}
%  
%\end{corollary}

\section{Experimental Results} \label{sec:res}
In this section, we present the experiments carried out  to assess the
practical effectiveness of the proposed approach.
The assessment is performed on the computation of  AXps and CXps
for d-DNNFs and SDDs.
The experiments consider a selection
of 11 binary datasets that are publicly available and originate from the
Penn Machine Learning Benchmarks~\cite{Olson2017PMLB} and the openML repository \cite{OpenML2013}.
To learn d-DNNFs (resp. SDDs), we first train Read-Once Decision Tree (RODT) models on the given
binary datasets and then compile the obtained RODTs into d-DNNFs (resp. SDDs).
(A RODT is a \textit{free} BDD (FBDD) whose underlying graph is a tree~\cite{barcelo-nips20,wegener-bk00},
where FBDD is defined as a BDD that satisfies the \textit{read-once property}:
each variable is encountered at most once on each path from the root to a leaf node.)
The compilation of RODTs to d-DNNFs can be easily done by direct mapping,
since RODT is a special case of FBDDs, and FBDDs is a subset of d-DNNFs~\cite{darwiche-jair02}
To compile SDDs,
we use the PySDD package\footnote{https://github.com/wannesm/PySDD},
which is implemented in Python and Cython. (Note that, we tuned PySDD to use
\textit{dynamic minimization}  \cite{choi2013dynamic} during the construction in order
to reduce the size of the  SDDs.)
The PySAT package \cite{imms-sat18} is used to instrument incremental SAT
oracle calls in  AXp/CXp enumeration.
Lastly, The experiments are performed on a MacBook Pro with a 6-Core Intel
Core~i7 2.6~GHz processor with 16~GByte RAM, running macOS Big Sur.

PySDD wraps the famous SDD package\footnote{http://reasoning.cs.ucla.edu/sdd/}
which offers canonical SDDs\footnote{%
  Since PySDD offers canonical SDDs, the \tbf{CD} transformation is
  not implemented in worst-case polynomial
  time~\cite{darwiche-aaai15}.
  However, in practice, this was never an issue in our experiments.}.
Employing canonical SDDs allows consistency and validity checking to
be done in constant time (If the canonical SDD is inconsistent
(resp. valid) then it is a single node labeled with $\bot$
(resp. $\top$)~\cite{darwiche-ijcai11}), so in practice may improve
the efficiency of explaining SDD classifiers.
\setlength{\tabcolsep}{5pt}
\rowcolors{2}{gray!10}{}
\let\lpr\undefined
\let\rpr\undefined
\newcommand{\lpr}{(}
\newcommand{\rpr}{)}

\begin{table*}[t]%[h]
\centering
\resizebox{\textwidth}{!}{
  \begin{tabular}{l>{\lpr}S[table-format=4.0,table-space-text-pre=\lpr]S[table-format=3.0,table-space-text-post=\rpr]<{\rpr}c cc  c ccc  ccc  cc cc}
\toprule[1.2pt]
\rowcolor{white}
\multirow{2}{*}{\bf Dataset} & \multicolumn{2}{c}{\multirow{2}{*}{\bf (\#F~~~~\#S)}}  & \multicolumn{3}{c}{\bf Model} & \multicolumn{1}{c}{\bf XPs}  & \multicolumn{3}{c}{\bf AXp} & \multicolumn{3}{c}{\bf CXp}  &  \multicolumn{2}{c}{\bf d-DNNF} & \multicolumn{2}{c}{\bf SDD}\\
  \cmidrule[0.8pt](lr{.75em}){4-6}
  \cmidrule[0.8pt](lr{.75em}){7-7}
  \cmidrule[0.8pt](lr{.75em}){8-10}
  \cmidrule[0.8pt](lr{.75em}){11-13}
  \cmidrule[0.8pt](lr{.75em}){14-15}
  \cmidrule[0.8pt](lr{.75em}){16-17} 
\rowcolor{white}
& \multicolumn{2}{c}{}   & {\bf \%A} & {\bf \#ND} & {\bf \#NS}  &  {\bf avg} & {\bf M} &  {\bf avg} & {\bf \%L} &  {\bf M} &  {\bf avg} & {\bf \%L} & {\bf M}  &  {\bf avg}   & {\bf M}  &  {\bf avg}  \\
\toprule[1.2pt]

corral & 6 & 160 & 100 & 35 & 12 &  4 & 4  & 2 & 34 & 4  & 2 & 22 & 0.004 &  0.001 & 0.001 & 0.000 \\
db-bodies & 4702 & 64 & 100 & 22 & 21 & 4 & 3  & 2 & 1 & 4  & 3 & 1 &  0.004 &  0.002  & 0.001 & 0.000 \\
db-bodies-stemmed & 3721 & 64 & 84.6 & 14 & 15  & 4 & 2  & 1 & 1 & 4  & 2 & 1 &  0.003  & 0.002   &  0.001 & 0.000\\
db-subjects & 242 & 64 & 84.6 & 45 &  28 & 6 & 4  & 2 & 2 & 6  & 4 & 1 &  0.005 &  0.002  & 0.004 & 0.001 \\
db-subjects-stemmed & 229  & 64 &  92.3 & 54 &  31 &  7 & 4  & 2 & 2 & 7  & 5 & 1 &  0.006 &  0.003   &0.004 & 0.001 \\
mofn\_3\_7\_10 & 10 & 1324 & 97.7 & 107 & 34 & 11 & 28  & 4 & 32 & 28  & 6 & 24 &  0.072 &  0.011  &  0.008 & 0.001\\
mux6 & 6 & 128 &  100 &  62 & 22 & 5 & 4  & 2 & 51 & 4  & 3 & 24 &  0.009 &  0.003   & 0.002 &  0.001\\
parity5+5 & 10 & 1124 &  85.7 & 484 & 96 & 9 & 12  & 2 & 66 & 19  & 7 & 14 &  0.193 &  0.038   & 0.009 &  0.002 \\
spect & 22 & 267 &  85.1 & 108  & 55 &  14 & 36  & 8 & 22 & 13  & 6 & 10 &  0.105 &  0.030   & 0.016 &  0.005 \\
threeOf9 & 9 & 512 &  96.1 & 76 & 37 & 7 & 15  & 3 & 38 & 14 & 4 & 19 &  0.023 &  0.006  & 0.005 &  0.001 \\
xd6 & 9 & 973 &  97.9 & 80 & 36 &  8 & 25  & 4 & 36 & 22  & 4 & 20 &  0.035 &  0.009   & 0.007 & 0.001\\

\bottomrule[1.2pt]
\end{tabular}
}
\caption{\footnotesize{Listing all  AXps  CXps for d-DNNFs and SDDs. 
Columns {\bf \#F} and {\bf \#S} report, resp., the number of features 
and the number of tested samples (instances), in the dataset. 
Sub-Column {\bf \%A} reports the (test) accuracy of the model and 
 {\bf \#ND}  (resp.\ {\bf \#NS}) shows the total number of nodes in 
 the  compiled d-DNNF (resp.\ SDD).
 Column {\bf XPs}  reports  the average number of total explanations 
(AXp's and CXp's). 
Sub-columns {\bf M} and {\bf avg} of column {\bf AXp} 
(resp.,  {\bf CXp}) show, resp., the maximum and average 
number of explanations. The average length of an explanation  (AXp/CXp) is given 
as {\bf \%L}.
Sub-columns  {\bf M} and {\bf avg}  of column {\bf d-DNNF} (resp. SDD) 
report, resp.,  maximal and average runtime (in seconds) 
to list all the explanations for all tested instances.
}}
\label{tab:dnnf-sdd}
\end{table*}

\autoref{tab:dnnf-sdd} summarizes the obtained results of explaining
d-DNNFs and SDDs.
(Note that, for each dataset, the compiled d-DNNF and SDD represent
the same decision function of the learned RODT.
Hence, the computed explanations are the same as well.)
Performance-wise,  the maximum running time to enumerate all AXps/CXps is less
than 0.2 sec for all tested d-DNNFs, and is less than 0.02 sec for all tested SDDs.
On average, it takes at most 0.038 sec for enumerating all the explanations (AXps/CXps) of d-DNNFs;
for SDDs, it takes a few milliseconds to enumerate all the explanations (AXps/CXps).
Thus the overall cost of the SAT oracle calls performed by
the enumeration algorithm is negligible.
Hence, it is plain that instrumenting SAT oracle calls does not constitute
a bottleneck to listing effectively all the AXps/CXps of the d-DNNFs and SDDs.
Apart from the runtime, one observation is that
the total number of AXps and CXps per instance is relatively small.
Moreover, if compared with the total number of features,
the average length of an explanation (AXp or CXp) is also
relatively small.
We compared the raw performance of explaining d-DNNFs and SDDs.
\autoref{fig:runtime} depicts a cactus plot showing the total runtime
spent on AXp-and-CXp enumeration for all instances of each dataset.
As can be seen, both d-DNNF and SDD explanation procedures are able to
finish successful enumeration of AXps/CXps for all instances of the
datasets in a few seconds.
Unsurprisingly, the runtimes in case of SDDs tend to be overall better
than those for d-DNNFs.
Indeed, explaining SDDs is on average 6 times faster than explaining
d-DNNFs.
One factor contributing to this performance difference is that in
practice in case of SDDs consistency and validity checking can be done
in constant time.

\begin{figure}[t]%[htb]
\centering
\includegraphics[width=0.6\textwidth]{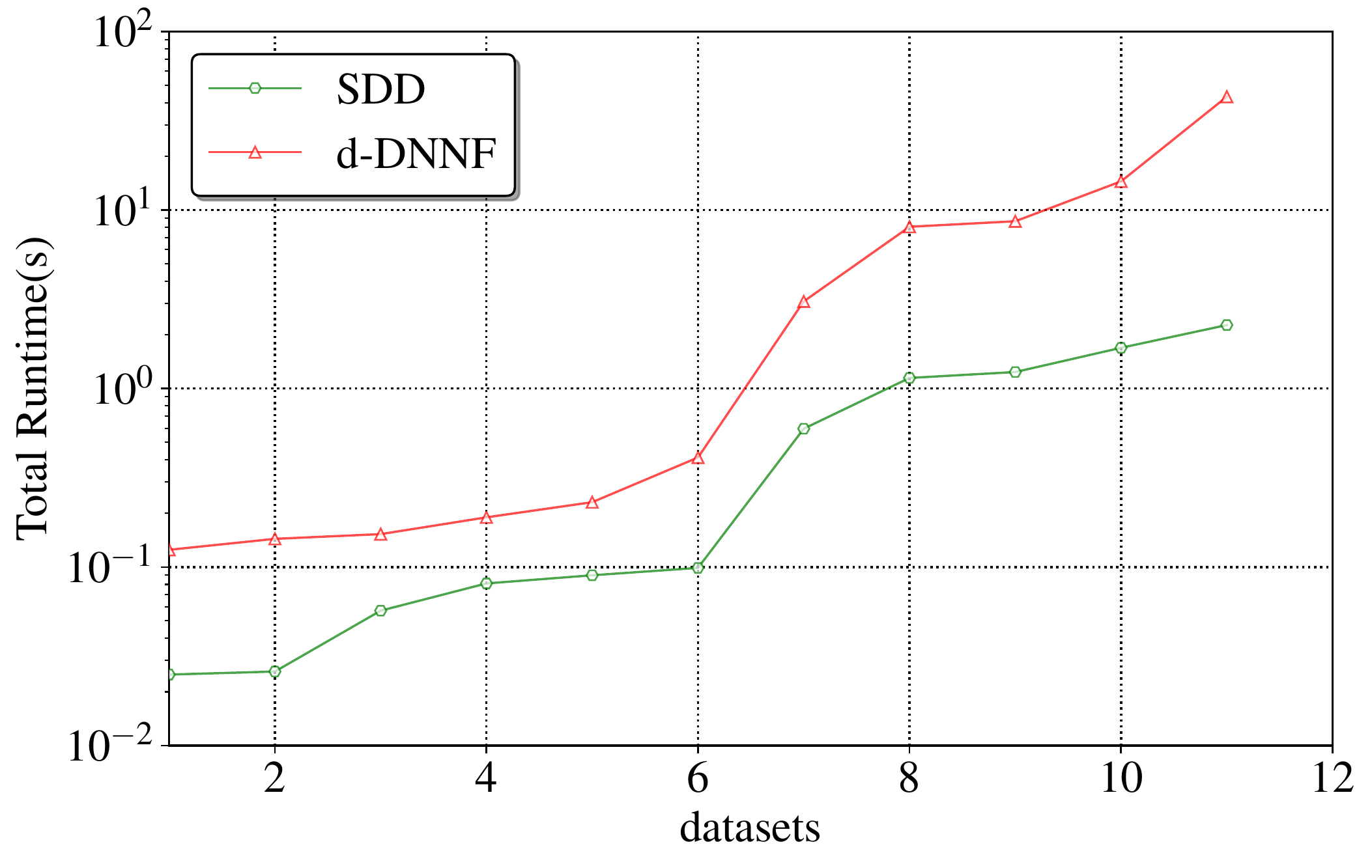}
\caption{Comparison of total runtime (in seconds) spent to explain all instances of each dataset for d-DNNFs vs.\ SDDs.}
\label{fig:runtime}
\end{figure}

To conclude, the results shown above, for the concrete case of
classifiers represented in the d-DNNF and SDD languages, support the
paper's theoretical claims from a practical side that, if the underlying KC language
implements polynomial-time \tbf{CO} and \tbf{VA} queries as well as
the \tbf{CD} transformation, then
(i)~the polynomial-time computation of one AXp/CXp in practice takes a
negligible amount of time, which together with
(ii)~making a single SAT oracle call per explanation makes
(iii)~the enumeration of (some/all) XPs (AXps and CXps) highly
efficient in practice.

\section{Conclusions}
\label{sec:conc}

This paper proves that for any classifier that can be represented with
a d-DNNF, both one AXp and on CXp can be computed in polynomial
time on the size of the d-DNNF. Furthermore, the paper shows that
enumeration of AXps and CXps can be implemented with one NP oracle
call per explanation. The experimental evidence confirms that for
small numbers of explanations, the cost of enumeration is negligible.
In addition, the paper proposes conditions for generalized decision
functions to be explained in polynomial time. Concretely, the paper
develops conditions which allow generalized decision functions
represented with DNNFs to be explainable in polynomial time.
Finally, the paper proposes general conditions for a classifier to be
explained in polynomial time.
The experimental results validate the scability of the polynomial time
algorithms and, more importantly, the scalability of oracle-based
enumeration.

% ...
%%\input{wplan}

%\clearpage
%%
%% Bibliography
%%
%% Please use bibtex, 
%\bibliography{refs,xai}

%\appendix

\end{document}